\documentclass[journal]{IEEEtran}
\AtBeginDocument{%
  }

\usepackage{url}
\usepackage{comment}
\usepackage{amsmath}
\usepackage{amssymb}

\usepackage{amsthm}
\usepackage{multirow}
\usepackage{array}
\usepackage{soul}
\usepackage{threeparttable}
\usepackage{makecell} 
\usepackage{setspace}
\usepackage{graphicx}
\newcolumntype{P}[1]{>{\centering\arraybackslash}m{#1}} 
\newcolumntype{L}[1]{>{\arraybackslash}m{#1}} 

\theoremstyle{definition} 
\newtheorem{definition}{Definition}[section] 

\theoremstyle{plain} 
\newtheorem{theorem}{Theorem}[section] 
\newtheorem{proposition}[theorem]{Proposition}

\theoremstyle{remark} 
\newtheorem*{remark}{Remark}

\makeatletter
\newcommand\footnoteref[1]{\protected@xdef\@thefnmark{\ref{#1}}\@footnotemark}
\makeatother

\newcommand{\miqing}[1]{\textcolor{red}{ MiqingComment: \textit{#1}}}

\begin{document}


\title{On the Problem Characteristics of Multi-objective Pseudo-Boolean Functions in Runtime Analysis}

\author{Zimin Liang and Miqing Li\(^*\)
\thanks{Zimin Liang and Miqing Li (\emph{corresponding author}) are with the School of Computer Science, University of Birmingham, Edgbaston, Birmingham B15 2TT, UK (emails: zxl525@student.bham.ac.uk; m.li.8@bham.ac.uk).}}

\maketitle

\begin{abstract}
Recently, there has been growing interest within the theoretical community in analytically studying multi-objective evolutionary algorithms. This runtime analysis-focused research can help formally understand algorithm behaviour, explain empirical observations, and provide theoretical insights to support algorithm development and exploration. However, the test problems commonly used in the theoretical analysis are predominantly limited to problems with heavy ``artificial'' characteristics (e.g., symmetric, homogeneous objectives and linear Pareto fronts), which may not be able to well represent realistic scenarios. In this paper, we first discuss commonly used multi-objective functions in the theory domain and systematically review their features, limitations and implications to practical use. Then, we present several new functions with more realistic features, such as heterogenous objectives, local optimality and nonlinearity of the Pareto front, through simply mixing and matching classical single-objective functions in the area (e.g., LeadingOnes, Jump and RoyalRoad). We hope these functions can enrich the existing test problem suites, and strengthen the connection between theoretic and practical research. 
\end{abstract}





\section{Introduction}

Many real-world challenges contain multiple conflicting objectives to be optimised simultaneously, known as multi-objective optimisation problems (MOPs). 
Instead of a single optimal solution, MOPs yield a set of solutions called \textit{Pareto optimal solutions}, each representing a distinct trade-off between the objectives. 
Multi-objective evolutionary algorithms (MOEAs) have demonstrated their ability to deal with MOPs in a variety of practical scenarios. 
Yet, the theoretical understanding of MOEAs, particularly their behaviours and performance guarantees, lags behind their practical success~\cite{zheng_theoretical_2023, knowles_evolutionary_2024}. 

In this regard, runtime analysis \cite{knowles_evolutionary_2024} emerges as a very useful tool. 
First, it can be used to theoretically analyse the expected runtime of MOEAs, including mainstream algorithms like NSGA-II~\cite{zheng_first_2022, bian_better_2022, doerr_first_2023}, SPEA2~\cite{ren2024first}, SMS-EMOA~\cite{Bian2023, zheng_how_2024, deng_runtime_2024}, MOEA/D~\cite{huang_running_2019, huang_runtime_2021, doerr_proven_2024} and NSGA-III~\cite{wietheger_mathematical_2024}, as well as simple heuristics like SEMO \cite{laumanns_running_2004, giel_effect_2006, wietheger_near-tight_2024} and G-SEMO \cite{giel_expected_2003, bian2018general, dang_crossover_2024}.
Second, runtime analysis can help confirm observations reported from empirical studies, for example, why NSGA-II is less effective for problems with three or more objectives~\cite{zheng_runtime_2024}.
Third, runtime analysis can provide insight and guidance in algorithm design (e.g., the use of crossover~\cite{dang_crossover_2024}, adaptive mutation~\cite{lehre_self-adaptation_2022} and the archive \cite{bian2024archive}), as well as in algorithmic parameter setup (e.g., mutation rate setting~\cite{lehre_negative_2010}).
Lastly, it can even be used to challenge conventional practice in the empirical community and guides the development of different mechanisms, for example introducing randomness in population update of MOEAs~\cite{bian2025stochastic,ren2025stochastic}.

Yet, multi-objective benchmark functions used in runtime analysis are predominantly limited to pseudo-Boolean ones with heavy ``artificial'' characteristics. 
They may not be able to well represent realistic scenarios. 
For example, in OneMinMax~\cite{giel_effect_2006}, a commonly used benchmark, any point in the decision space is a Pareto optimal solution.
That means the two objectives are completely conflicting, with no configuration improving both objectives or yielding dominated solutions. This is apparently not very realistic.

Another prominent feature in multi-objective pseudo-Boolean functions is the linear shape of their Pareto fronts, such as in the well-known benchmarks LeadingOnes-TrailingZeroes~\cite{laumanns_running_2002}, Count- 
Ones-CountZeroes~\cite{laumanns_running_2004} 
and OneJump-ZeroJump~\cite{doerr_theoretical_2021-1}. 
This may not be very representative of real-world problems and, more importantly, it can favour certain MOEAs (e.g., the decomposition-based algorithm MOEA/D \cite{Zhang2007}).
The performance of some practical MOEAs may heavily depend on the Pareto front shape due to their population maintenance mechanisms \cite{Li2016}. 
For example, MOEA/D uses a set of uniformly distributed weights in a simplex to maintain the population. 
As a result, the algorithm may achieve perfect uniformity when the function's Pareto front is of a linear or simplex-like shape, but may perform very poorly on other shapes \cite{Ishibuchi2016}. 
Solely considering such linear Pareto fronts may fail to provide a fair comparison between MOEAs, particularly to those who are robust to Pareto front shapes such as Pareto-based MOEAs \cite{Li2016}. 

Findings derived from very artificial functions may not hold on practical problems. For example, to our knowledge, all existing theoretical studies suggest that NSGA-II performs as least as well as SEMO, and sometimes better. However, recent empirical work \cite{li2024empirical} has shown that SEMO can outperform NSGA-II on many practical problems (e.g., the multi-objective TSP and QAP problems).

In this paper, we conduct a short review of multi-objective pseudo-Boolean functions commonly used in runtime analysis.
We discuss the characteristics that make these functions useful for theoretical study, and identify their limitations in capturing realistic features, as well as the implications for the practical community of algorithm development.
Moreover, we present several new functions by combining different single-objective functions such as LeadingOnes, Jump, and RoyalRoad.
These mix-and-match functions, aside from being composed of heterogenous objectives on which MOEAs may perform differently \cite{cosson2024bi}, can introduce more realistic features such as multi-modality (i.e., Pareto local optimality) and nonlinearity of the Pareto front, while still retaining analytical tractability.

It is worth stating that in this paper we focus on widely used, representative multi-objective pseudo-Boolean functions in runtime analysis, including OneMinMax, LeadingOnes-TrailingZeroes, CountOnes-CountZeroes and OneJump-ZeroJump, as well as ``asymmetric'' combinations of their single-objective components such as OneMax~\cite{muhlenbein_how_1992}, LeadingOnes~\cite{droste_analysis_2002}, Jump~\cite{jansen_analysis_1999} and RoyalRoad~\cite{mitchell_royal_1992}.
There do exist many other benchmarks in the area such as OneTrap-ZeroTrap~\cite{dang_illustrating_2024}, LPTNO~\cite{qian_analysis_2013}, SPG, ZPLG~\cite{qian_selection_2016} (extend from PLG~\cite{friedrich_illustration_2011}), Dec-obj-MOP~\cite{li_primary_2016} and multi-objective RealRoyalRoad~\cite{dang_crossover_2024}, and single-objective ones such as the Needle~\cite{oliveto_runtime_2011}, Trap~\cite{oliveto_runtime_2011} and RealRoyalRoad~\cite{jansen_real_2005}. 
However, 
such functions may be either too special (e.g., the global optimum in the Needle and Trap functions too hard to locate), or are regarded too artificial~\cite{doerr_theoretical_2021-1}.

The rest of the paper is structured as follows. Section II gives the preliminaries in multi-objective optimisation and single-objective pseudo-Boolean functions considered in the paper. 
Section III describes the function characteristics of interest. 
Sections IV and V discuss existing common benchmarks and new mix-and-match functions based on the aforementioned characteristics and their implications, followed by section VI that gives the limitations of the study.
Finally, Section VII concludes the paper. 

\section{Preliminaries}\label{sec:preliminaries}
In this section, we will first introduce basic concepts in multi-objective optimisation. We will then describe several well-known single-objective pseudo-Boolean functions used in runtime analysis and their extension to multi-objective pseudo-Boolean functions.

\subsection{Multi-Objective Optimisation}
Let \(f(x)=(f_1(x), f_2(x))\) be a bi-objective (maximisation) optimisation problem in decision space \(\Omega\) for all \(x\in\Omega\) that \(f_i(x):\Omega\to\mathbb{R}\).
As we only consider Boolean functions, each $x$ is a bit-string \(\{0,1\}^n\). 

Considering two solutions \(x,x'\in\Omega\), we say \(x\) weakly (Pareto) dominates \(x'\), denoted \(x\succeq x'\), if \(f_1(x) \geq f_1(x')\) and \(f_2(x) \geq f_2(x')\); \(x\) dominates \(x'\) if one of the inequalities is strict, denoted \(x\succ x'\).
A solution \(x\) is \textit{Pareto optimal} if no other solution in $\Omega$ dominates it.
The set of all such solutions is called the \textit{Pareto set} in the decision space, and their image under \(f\) is the \textit{Pareto front} (objective space).

\subsection{Single-objective pseudo-Boolean Functions and Their Multi-objective Extensions}

In single-objective runtime analysis, a range of pseudo-Boolean functions have been studied. We focus on four representative ones: OneMax, LeadingOnes, Jump, and RoyalRoad.

\subsubsection{OneMax}
OneMax counts the number of \(1\)s in a bit-string.
\begin{definition} (\textit{OneMax}~\cite{muhlenbein_how_1992}). 
Let \( x \in \{0,1\}^n \) be a bit-string of length \( n \). The OneMax function is defined as:
\begin{align}
\text{OneMax}(x) = \sum_{i=1}^{n} x_i.
\end{align}
\end{definition}

Its symmetric counterpart is the OneMin (ZeroMax) function, as the second objective function of the OneMinMax benchmark.

\begin{definition} (\textit{OneMinMax} (OMM)~\cite{giel_effect_2006}).
Let \( x \in \{0,1\}^n \) be a bit-string of length \( n \). The problem has two objectives:
\begin{align}
f_1(x) = \sum_{i=1}^{n} x_i \quad \text{and} \quad f_2(x) = \sum_{i=1}^{n} (1 - x_i).
\end{align}
\label{eq:oneminmax}
\end{definition}
The goal of the function is to simultaneously maximise both \( f_1(x) \) and \( f_2(x) \), which is inherently contradictory, as improving one objective results in decreasing the other.

\subsubsection{LeadingOnes}
This problem counts the number of consecutive ones starting from the leftmost side of the bit-string.
\begin{definition} (\textit{LeadingOnes}~\cite{rudolph1997convergence}).
Let \( x \in \{0,1\}^n \) be a bit-string of length \( n \). The LeadingOnes function is defined as:
\begin{align}
\text{LeadingOnes}(x) = \sum_{i=1}^{n} \prod_{j=1}^{i} x_j.
\end{align}
\end{definition}

The symmetric counterpart of LeadingOnes is the TrailingZeroes function, as the second objective function of the LeadingOnes-TrailingZeroes benchmark in the following.

\begin{definition} (\textit{LeadingOnes-TrailingZeroes} (LOTZ) \cite{laumanns_running_2002}).
Let \( x \in \{0,1\}^n \) be a bit-string of length \( n \). The problem is defined as:
\begin{align}
f_1(x) = \sum_{i=1}^{n} \prod_{j=1}^{i} x_j \quad \text{and} \quad f_2(x) = \sum_{i=1}^{n} \prod_{j=i}^{n} (1 - x_{j}).
\end{align}
\end{definition}
The goal here is to maximise both \( f_1(x) \) and \( f_2(x) \), which encourages solutions with \(1\)s in the beginning and \(0\)s at the end.

\subsubsection{Jump}
The Jump function (also known as the $k$-Jump function) is a variation of the OneMax function but involves jumps (across a valley) in the fitness landscape. The objective function penalises solutions that are close to the optimum (the penalty becoming severer as getting closer). Formally, it is defined as follows.
\begin{definition} (\textit{$k$-Jump (OneJump)~\cite{jansen_analysis_1999}}).
Let \( x \in \{0,1\}^n \) be a bit-string of length \( n \) with a jump parameter \( k\in\{1,2,\dots,n\}\) indicating the size of the valley. The $k$-Jump function is formalised as:
\begin{align}
\text{Jump}_k(x) = 
\begin{cases}
k + |x|_1, & \text{if }  |x|_1 \leq n-k\,or\,x=1^n \\
n - |x|_1 , & \text{otherwise}.
\end{cases}
\end{align}
\label{eq:jump}
\end{definition}

In the area of multi-objective optimisation, the jump function and its symmetric counterpart --- ZeroJump function --- composed of the OneJump-ZeroJump benchmark, as in the following.

\begin{definition} \label{def:OJZJ}(\textit{OneJump-ZeroJump} (OJZJ)~\cite{doerr_theoretical_2021-1}).
Let \( x \in \{0,1\}^n \) be a bit-string of length \( n \) and let \( k \) be a fixed jump parameter with \( 1 \leq k < \frac{n}{2} \). The OneJump-ZeroJump problem is defined as:
\begin{align}
\begin{split}
f_1(x) &= 
\begin{cases}
k + |x|_1, & \text{if }  |x|_1 \leq n-k\,or\,x=1^n \\
n - |x|_1 , & \text{otherwise}
\end{cases} \\
f_2(x) &= 
\begin{cases}
k + |x|_0, & \text{if }  |x|_0 \leq n-k\,or\,x=0^n \\
n - |x|_0 , & \text{otherwise}.
\end{cases}
\end{split}
\end{align}
\end{definition}

Here, \( f_1 (x) \) prefers bit-strings with a larger number of ones until it gets close to \( 1^n \), in which case the objective value drops. Similarly, \( f_2(x) \) prefers strings with a larger number of zeros until it is too large and close to \( 0^n \). This forms two symmetric and (partially) conflicting objectives, which we will explain in detail in Section IV.

\subsubsection{RoyalRoad}
The RoyalRoad function is designed to provide a fitness landscape with explicit building blocks which introduce neutral areas.
\begin{definition} (\textit{RoyalRoad}~\cite{mitchell_royal_1992}).
Let \( x \in \{0,1\}^n \) be a bit-string of length \( n \), partitioned into \( b \) disjoint blocks \( S_1, S_2, \dots, S_b \), each with the same length \( \ell \) (where \( n = b \cdot \ell \), \(b>1\)). The RoyalRoad function is defined as:
\begin{align}
\text{RoyalRoad}(x) = \sum_{j=1}^{b} \bigl(\ell\prod_{i\in S_j} x_i \bigr).
\end{align}
\end{definition}

Likewise to OneJump, we regard the RoyalRoad function here as OneRoyalRoad, its symmetric counterpart is ZeroRoyalRoad, as the second objective function of the OneRoyalRoadZeroRoyalRoad benchmark in the following.

To our knowledge, the RoyalRoad function has not been considered and studied in multi-objective optimisation\footnote{Note that there is a similar but different problem, called RealRoyalRoad~\cite{jansen_real_2005}, has been extended and studied recently in multi-objective optimisation \cite{dang_crossover_2024, bian2025stochastic, opris_many_2024}.}, though it can be trivially extended to by optimising the OneRoyalRoad and ZeroRoyalRoad functions. 

\begin{definition}\label{sec:royalroad} (\textit{OneRoyalRoad-ZeroRoyalRoad (ORZR)}).
Let \( x \in \{0,1\}^n \) be a bit-string of length \( n \), partitioned into \( b \) disjoint blocks \( S_1, S_2, \dots, S_b \), each with  the same length \( \ell \) (where \( n = b \cdot \ell \), \(b>1\)). The problem is:
\begin{align}
\begin{split}
f_1(x) = \sum_{j=1}^{b} \bigl(\ell\prod_{i \in S_j} x_i\bigr), \quad \text{and} \quad
f_2(x) = \sum_{j=1}^{b} \bigl(\ell\prod_{i \in S_j} (1 - x_i)\bigr).
\end{split}
\end{align}
\end{definition}

Here, \( f_1(x) \) prefers strings that have all bits set to one in each block, while \( f_2(x) \) prefers strings that have all bits set to zero in each block, forming two symmetric and conflicting objectives.


\section{Characteristics of Multi-objective Optimisation Problems}\label{sec:features}

In this section, we will review representative characteristics of multi-objective optimisation problems. These include characteristics that existing multi-objective pseudo-Boolean functions commonly have, and characteristics that practical multi-objective optimisation problems typically have. 
In doing so, we try to explain what having (or not having) these characteristics may imply.

\subsection{Completely Conflicting Objectives}

An interesting characteristic resulting from the symmetric extension of single-objective pseudo-Boolean functions is that some multi-objective functions may have rather conflicting objectives. An extreme example is the OneMinMax function~\cite{giel_effect_2006} where the two objectives are completely conflicting. That is, an improvement on one objective always leads to a deterioration on the other. 
This characteristic can be defined as follows.
\begin{definition}(\textit{Completely conflicting objectives}).
\begin{align}
\begin{split}
\text{For all } x, x' \in \Omega: \quad f_1(x') < f_1(x) \implies f_2(x') > f_2(x) \\ \text{and} \quad f_2(x') < f_2(x) \implies f_1(x') > f_1(x).
\end{split}
\end{align}
\end{definition}

When a multi-objective optimisation problem has this characteristic, all solutions in the decision space are Pareto optimal ones. This apparently is not very realistic, though there do exist some early continuous benchmark functions having the characteristic, such as SCH~\cite{Veldhuizen1999}. 
From the perspective of MOEAs, this type of problem does not test an algorithms' ability to converge, but rather its ability to distribute solutions over the Pareto front. As such, some diversity/novelty-driven search~\cite{antipov_rigorous_2023, neumann_discrepancy-based_2018} strategies may have the edge in addressing such problems.

\subsection{Symmetry}
Another characteristic resulting from the straightforward extension of single-objective pseudo-Boolean functions is that the resultant multi-objective problems are symmetric, meaning that the objectives are similar in form or mirror each other.
For pseudo-Boolean problems, this means reversing each bit as well as their positions.
Formally, it can be defined as follows.
\begin{definition}\label{def:symm} (\textit{symmetry in pseudo-Boolean functions}).
Let \(f:\{0,1\}^n\to\mathbb{R}\) and \(g:\{0,1\}^n\to\mathbb{R}\) be two functions on binary strings of length \(n\).
\(f\) and \(g\) are \textbf{symmetric} if
\begin{align}
f(x)=g(B(R(x))) \quad \text{and} \quad g(x)=f(B(R(x)))
\end{align}
where \(R(x)\) is a function that reverses the order of bit-string \(x\) from \(\{1,2,\dots,n\}\) to \(\{n,\dots,2,1\}\); \(B(x)\) is a bitwise complementation function that flips all bits (i.e., \(x_i = 1 - x_i\) for \(i \in \{1,2,\dots,n\}\)).
\end{definition}

Crossover may benefit from such a symmetric setup. For example, if the all-zero solution is a non-dominated solution, then the all-one solution must be a non-dominated solution as well, hence being preserved highly likely in the search process of an MOEA. As all the bits of these two solutions are different, crossing them over may generate any bit-string, potentially speeding up the search. 

Another implication of having the symmetric setup is that solutions in the objective space are symmetric with respect to the line $f_1=f_2$. Such symmetry is rare in real-world scenarios, where objectives often differ in scale and structure. For example, in the multi-objective capacitated vehicle routing problem~\cite{Jozefowiez2008}, objectives like total distance and vehicle load balance are inherently heterogeneous and asymmetric. As for the effect on algorithms, some types of MOEAs are sensitive to problems with different objective ranges, such as indicator-based and decomposition-based algorithms~\cite{Li2016,ishibuchi2017effect,he2023effects}. 

\subsection{Disjoint Pareto Optimal Solutions}
Disjointness means that not all the optimal solutions (in decision space) are connected.
Connectedness can be defined using the notion of \textit{neighbourhood}, with an intuitive assumption that neighbouring solutions are connected to each other by default.
In terms of pseudo-Boolean functions, neighbourhood can be defined by one-bit difference of solutions.
\begin{definition} (\textit{Neighbourhood}).
The neighbourhood of a solution \(x\) is
\begin{align}
\mathcal{N}(x)=\{ x'\in\Omega \mid \text{Hamming}(x,x')=1 \}.
\end{align}
\end{definition}

Consider the set of Pareto optimal solutions as a directed graph where a node is a solution and the edge is defined by the neighbourhood relationship, this set is disjoint 
if the graph is disconnected (i.e., there exists no path between some pairs of nodes). Formally, it can be defined as follows.

\begin{definition}\label{def:connect}(\textit{Disjoint Pareto optimal solutions}).
For a set of Pareto optimal solutions \(\mathcal{PS}\), denote a graph 
\(G_{\mathcal{PS}}=(\mathcal{PS},E)\), where a node corresponds to a solution in $\mathcal{PS}$, and \(E=\{(x,x')\mid x,x'\in\mathcal{PS} \text{ and } x'\in \mathcal{N}(x)\}\). The problem has disjoint Pareto optimal solutions if \(G_{\mathcal{PS}}\) is disconnected.
\end{definition}

In many existing pseudo-Boolean functions for runtime analysis, Pareto optimal solutions are not disjoint, hence forming a connected region, such as in OneMinMax~\cite{giel_effect_2006} and LeadingOnes-TrailingZeroes~\cite{laumanns_running_2002}. 
In this case, local search can easily reach all Pareto optimal solutions, provided that it hits one of them.
However, many practical combinatorial problems have been shown that their Pareto optimal regions are disjoint, such as multi-objective linear programming~\cite{ehrgott1997connectedness}, spanning tree~\cite{ehrgott1997connectedness} and knapsack~\cite{gorski2011connectedness}. 
This can also be echoed by the multi-funnel structure in multi-objective combinatorial problems~\cite{ochoa2024funnels}.
As for algorithm design, having disjoint Pareto optimal solutions may highlight the importance of maintaining the diversity of solutions in decision space, rather than purely in objective space. 
Unfortunately, mainstream MOEAs only consider the latter, and performance of such algorithms can be improved substantially through considering the diversity of solutions in decision space~\cite{shir2009enhancing,cuate2019variation,ren2024maintaining}.

\subsection{Pareto Local Optimality}
Local optima are commonly observed in real-world problems, which may cause the search of an optimisation algorithm to be stuck. 
In single-objective optimisation with fitness function \(f:\Omega\to\mathbb{R}\), a solution \(x\) is a local optimum if, for some neighbourhood \(\mathcal{N}(x)\subset\Omega\), we have \(\nexists x'\in\mathcal{N}(x) \textrm{ such that } f(x')>f(x)\), meaning that its fitness is not worse than any of its neighbours. There are some single-objective pseudo-Boolean functions having local optimal solutions, such as $k$-Jump~\cite{jansen_analysis_1999}.

Similarly, in the multi-objective case, a solution is a Pareto local optimum if it is not dominated by any of its neighbours~\cite{paquete_local_2007}.
\begin{definition}\label{def:PLO}(\textit{Pareto local optimality}).
A solution \(x\) is a Pareto local optimum if 
\begin{align}
\nexists x'\in\mathcal{N}(x)\;\text{such that}\;x'\succ x.
\end{align}
\end{definition}

Note that global optimal solutions are also local optimal solutions by definition, but for the convenience of discussion, we exclude them when discussing local optimal solutions.

Pareto local optimality is an important characteristic that many practical problems have, and it can significantly increase the level of hardness for MOEAs to deal with. For example, in the DTLZ suite \cite{Deb2005a}, the only difference between DTLZ2 and DTLZ3 is that the latter has a number of local Pareto fronts, which imposes a much bigger challenge for MOEAs to converge~\cite{Li2013b}.

However, many well-established multi-objective pseudo-Boolean functions in runtime analysis do not have Pareto local optimal solutions, though some of their single-objective versions have.
For example, OneJump-ZeroJump~\cite{doerr_theoretical_2021-1}, the extension of $k$-Jump, does not have any Pareto local optimal solution. The local optimal solutions in $k$-Jump become non-dominated solutions (i.e., global optimal solutions in the multi-objective context) as no solution is better than them on both objectives. We will explain it in the section (Section IV-C).

\subsection{Separability of Variables}
Separability refers to how variables in the problem interact with one another with respect to the objective function.
\begin{definition}\label{def:separable} (\textit{Separability}~\cite{droste_rigorous_1998}).
An objective function \( f(x) \) with \( x\in\Omega\) is \textbf{separable} if it can be expressed as:
\begin{align}
f(x) = g_1(x_1) + g_2(x_2) + \dots + g_n(x_n)
\end{align}
where \( g_i \) is a function of the single variable \(x_i\). 
\end{definition}

The definition of separability is strict, requiring every variable of the problem is separable. With such a characteristic, an optimiser can optimise the problem variable by variable~\cite{droste_rigorous_1998, brockhoff_using_2022}.
However, not many real-world problems meet this condition. 
Instead, they may be partially separable, i.e., some of the variables can be optimised separably. 
For example, in many ZDT functions~\cite{Zitzler2000}, variables can be separated as the first variable \(x_1\) alone controls the first objective and the remaining variables contribute to the second objective.

\begin{table*}[tb]
    \centering
    \begin{threeparttable}
    \caption{Characteristics of practical multi-objective optimisation problems and whether existing and new multi-objective pseudo-Boolean functions in runtime analysis have them. 
    }
    \label{tab:features}        
    \scriptsize
    \begin{tabular}{@{}P{4cm}@{}|@{}P{1.6cm}@{}|@{}P{1.65cm}@{}|@{}P{1.6cm}@{}|@{}P{1.6cm}@{}|@{}P{1.6cm}@{}|@{}P{1.6cm}@{}|@{}P{1.6cm}@{}|@{}L{2.4cm}@{}}
    
        \hline
        \shortstack{Problem} & \shortstack{Non-symmetric\\objectives} & \shortstack{Non-completely\\ conflicting\\objectives} & \shortstack{Disjoint\\optimal\\solutions} & \shortstack{Not fully\\separable\\variables} & \shortstack{Low ratio of\\Pareto optimal\\solutions} & \shortstack{Non-linear\\Pareto front\\shape} & \shortstack{Pareto local\\optimality} & \shortstack{Studied in} \\ 
        \hline
        \shortstack{OneMinMax (OMM)} &  &  &  &  &  &  & & \cite{giel_effect_2006, nguyen_population_2015, doerr_runtime_2016, zheng_first_2022, zheng_better_2022, zheng_mathematical_2023, doerr_understanding_2023, zheng_runtime_2024-1, bian2024archive, ren2024maintaining, doerr_proven_2024, ren2024first} \\
        \hline
        \shortstack{LeadingOnes-TrailingZeroes (LOTZ)} & & \checkmark  &  & \checkmark & \checkmark &  & & \cite{laumanns_running_2002, giel_expected_2003,laumanns_running_2004, brockhoff_analyzing_2008, qian_analysis_2013, nguyen_population_2015, bian2018general, bian_better_2022, zheng_first_2022, zheng_better_2022, zheng_mathematical_2023, dang_analysing_2023, bian2024archive, ren2024maintaining, ren2024first} \\
        \hline
        \shortstack{OneJump-ZeroJump (OJZJ)} & & \checkmark & \checkmark & \checkmark & \checkmark\tnote{1} &  & & \cite{doerr_theoretical_2021-1, doerr_first_2023, zheng_theoretical_2023, doerr_run_2023, Bian2023, ren2024first} \\
        \hline
        \shortstack{CountOnes-CountZeroes (COCZ)} & \checkmark & \checkmark  &  &  & \checkmark &  & & \cite{laumanns_running_2004, qian_analysis_2013, li_primary_2016, bian2018general, huang_running_2019} \\
        \hline
        \shortstack{OneRoyalRoad-ZeroRoyalRoad (ORZR)} & &  \checkmark & \checkmark & \checkmark & \checkmark &  & \checkmark\tnote{2} & \\
        \hline 
        \hline
        \shortstack{OneMax-TrailingZeroes (OMTZ)} & \checkmark & \checkmark &  & \checkmark & \checkmark &  & & \\
        \hline
        \shortstack{OneMax-ZeroJump (OMZJ)} & \checkmark & \checkmark & \checkmark & \checkmark &  &  &  &\\
        \hline
        \shortstack{OneMax-ZeroRoyalRoad (OMZR)} & \checkmark & \checkmark & \checkmark & \checkmark & \checkmark &  & & \\
        \hline
        \shortstack{LeadingOnes-ZeroJump (LOZJ)} & \checkmark & \checkmark & \checkmark & \checkmark & \checkmark &  & \checkmark & \\
        \hline
        \shortstack{LeadingOnes-ZeroRoyalRoad (LOZR)} & \checkmark & \checkmark & \checkmark & \checkmark & \checkmark &  & \checkmark & \\
        \hline
        \shortstack{OneJump-ZeroRoyalRoad (OJZR)} & \checkmark & \checkmark & \checkmark & \checkmark & \checkmark & \checkmark\tnote{3} & \checkmark & \\
        \hline
    \end{tabular}
    \begin{tablenotes}
        \scriptsize
        \item[1] The ratio of Pareto optimal solutions is not low when the jump parameter \( k\leq\frac{n}{2}-\sqrt{\frac{n\ln4}{2}}\) (See Prop.~\ref{prop:OJZJ-k} in the Appendix).
        \item[2] Pareto local optima exist when block length \( \ell > 3 \).
        \item[3] The shape of the Pareto front is generally non-linear, except for the case that \( (n-k)\mod \ell = 0 \).
    \end{tablenotes}
    \end{threeparttable}
\end{table*}

\subsection{Ratio of Pareto Optimal Solutions}

In real-world problems, Pareto optimal solutions usually only take a small portion of the search space.
However, 
in some multi-objective pseudo-Boolean functions in runtime analysis (e.g., OneMinMax and OneJump-ZeroJump), most of their solutions are Pareto optimal ones. 
For a pseudo-Boolean function, the ratio of Pareto optimal solutions is \(\frac{|\mathcal{PS}|}{2^n}\), where $PS$ denotes the set of Pareto optimal solutions. 
We say a function has a low ratio if it converges to 0 as \(n\) grows to infinity. 
Most practical problems have a low ratio, even though the number of the Pareto optimal solutions can grow exponentially with the problem size~\cite{Ehrgott2006}, such as in bi-objective versions of the shortest path problem~\cite{hansen_bicriterion_1980}, minimum spanning tree problem~\cite{hamacher_spanning_1994} and integer minimum-cost flow problem~\cite{ruhe_complexity_1988}.

\subsection{Shape of the Pareto Front}
Real-world multi-objective optimisation problems may have a very complex geometry of their Pareto front shapes, such as various convex, concave, mixed, and/or disjoint fronts. 
However, existing multi-objective pseudo-Boolean functions in runtime analysis usually have a simple linear Pareto front.
For MOEAs which aim to approximate a good representation of a problem's Pareto front with a fixed population size, the shape of the problem matters. Some types of MOEAs favour linear Pareto front, such as hypervolume-based algorithms \cite{Auger2009} and decomposition-based algorithms \cite{Ishibuchi2016}. 
In contrast, Pareto-based algorithms are robust to Pareto front shapes \cite{Li2016}. 
This may be one of the reasons that NSGA-II may generally perform better than SMS-EMOA \cite{Beume2007} and MOEA/D \cite{Zhang2007} in classic combinatorial problems \cite{li2024empirical} and real-world problems \cite{ishibuchi2023performance}.

\begin{figure*}[htbp]
\begin{center} 
\small
    \includegraphics[scale=0.52]{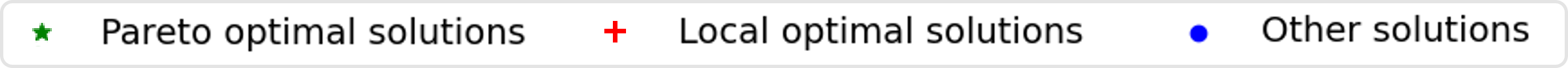}
    \renewcommand{\arraystretch}{0.1}
    \begin{tabular}{@{}c@{}c@{}c}
		\includegraphics[scale=0.45]{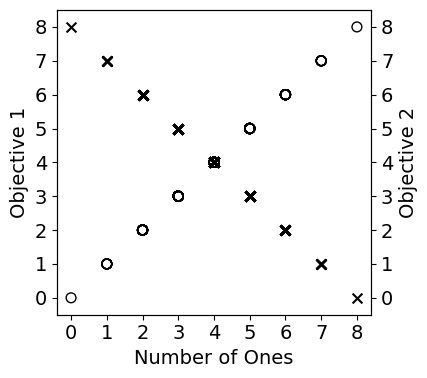}~&~
		\includegraphics[scale=0.45]{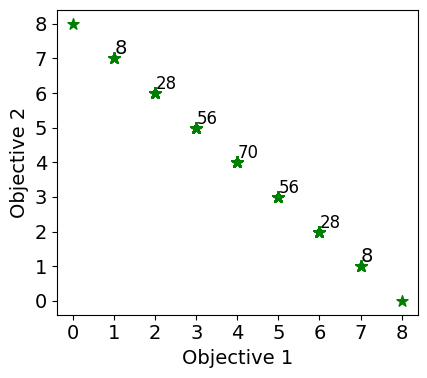}~&~
        \includegraphics[scale=0.45]{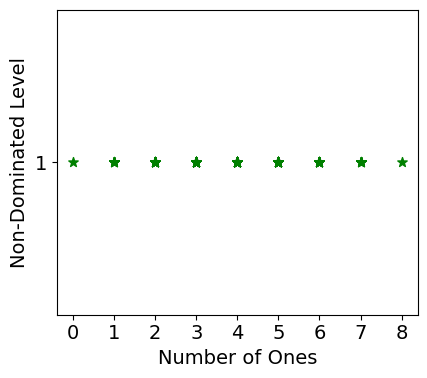}\\
		(a) Two individual objectives wrt \#\(1\) ~&~ 
		(b) Objective space ~&~
            (c) Nondominated level wrt \#\(1\) 
    \end{tabular}
    \end{center}

\caption{\footnotesize \textbf{OneMinMax (OMM)} ($n=8$) 
(a) Two individual objectives (OneMax and OneMin) of the OMM problem with respect to the number of ones. (b) Objective space, where the number associated with a solution means how many solutions in the decision space map to that solution. (c) The level of solutions with respect to the number of ones based on the Pareto non-dominated sorting~\cite{Goldberg1989}.
In (b) and (c), green, red and blue points indicate Pareto optimal solutions, local optimal solutions, and other solutions respectively. In this problem, all solutions are Pareto optimal.
}

\label{fig:OMM}
\end{figure*}

\section{Review of Multi-objective Boolean Functions}
In this and the next sections, we will review multi-objective pseudo-Boolean functions based on the characteristics described in Section~\ref{sec:features}. This section covers the well-known benchmarks such as OneMinMax~\cite{giel_effect_2006}, LeadingOnes-TrailingZeroes~\cite{laumanns_running_2002}, OneJump-Zero-\\Jump~\cite{doerr_theoretical_2021-1} and CountOnes-CountZeroes~\cite{laumanns_running_2004}; the next section introduces new mix-and-match benchmarks of the single-objective functions.
Table~\ref{tab:features} summarises how these functions reflect the key characteristics of practical multi-objective optimisation problems. 


\subsection{OneMinMax (OMM)} 

OMM \cite{giel_effect_2006} is to simultaneously maximise the number of ones and zeros in a bit-string, defined in Definition~\ref{eq:oneminmax}. It has been commonly used in runtime analysis (see Table~\ref{tab:features}).
OMM is arguably one of the least realistic multi-objective pseudo-Boolean functions.
As can be seen in Table~\ref{tab:features}, OMM has none of the characteristics that practical multi-objective problems may have.

Figure~\ref{fig:OMM} illustrates the 8-bit OMM function from three different perspectives. Specifically, Figure~\ref{fig:OMM}(a) plots the solutions on two individual objectives (OneMax and OneMin) with respect to the number of ones. Figure~\ref{fig:OMM}(b) plots the objective space, where the number associated with a solution means how many solutions in the decision space map to that solution. 
Figure~\ref{fig:OMM}(c) plots the non-dominated levels \cite{Goldberg1989} of solutions with respect to the number of ones, which helps indicate how easy solutions may move into a better place with respect to the Pareto non-dominated levels (similar to \textit{Pareto Landscape} \cite{liang2024pareto}).

As can be seen from Figure~\ref{fig:OMM}(a), the two objectives OneMax and OneMin are completely conflicting, making all solutions Pareto optimal (Figure~\ref{fig:OMM}(b) and (c)).
In addition, the objectives in OMM are fully separable with respect to their variables, allowing MOEAs to optimise the problem one variable (bit) at a time.

\subsection{LeadingOnes-TrailingZeroes (LOTZ)}

LOTZ~\cite{laumanns_running_2002} is to simultaneously maximise the number of consecutive ones from left to right and the number of consecutive zeroes from right to left.
It is one of the earliest and most studied multi-objective pseudo-Boolean functions.
As can be seen in Table~\ref{tab:features}, LOTZ has some unrealistic characteristics such as symmetric objectives, a linear Pareto front, no disjoint optimal solutions nor local optima, (see Figure~\ref{fig:LOTZ}). The last two may make the problem generally easy to solve (i.e., finding all the optimal solutions).

Unlike OneMinMax, the LOTZ function exhibits characteristics 
such as non-completely conflicting objectives and a low ratio of Pareto optimal solutions.
As can be seen from Figure~\ref{fig:LOTZ}(a),
the presence of dominated solutions indicates that the objectives are not completely conflicting.
The number associated with a point in Figure~\ref{fig:LOTZ}(b) stands for the number of solutions with identical objective values (a point without a number is a unique solution). 
Clearly, most of the solutions are dominated ones and there are only $n+1$ solutions in the Pareto front. 
In addition, unlike OMM, the objectives in LOTZ are not fully separable, implying that MOEAs can no longer optimise them independently on a bit-wise basis.


\begin{figure*}[htbp]
\begin{center} 
\small
    \includegraphics[scale=0.52]{figures/legend_horizontal.png}
    \renewcommand{\arraystretch}{0.1}
    \begin{tabular}{@{}c@{}c@{}c}
		\includegraphics[scale=0.45]{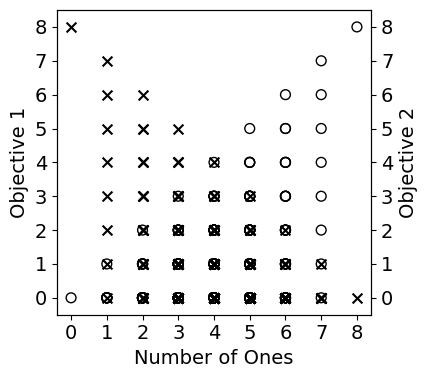}~&~
		\includegraphics[scale=0.45]{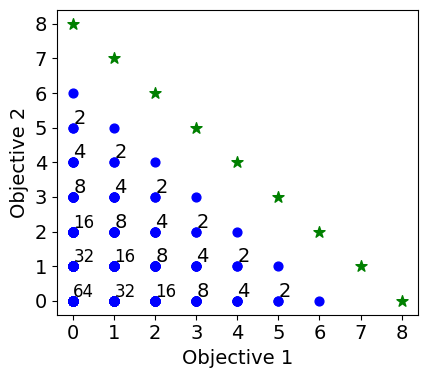}~&~
        \includegraphics[scale=0.45]{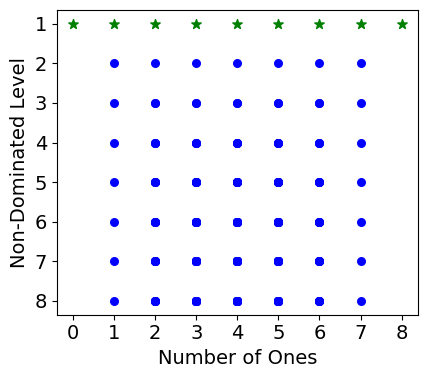}\\
		(a) Two individual objectives wrt \#\(1\) ~&~ 
		(b) Objective space ~&~
        (c) Nondominated level wrt \#\(1\) 
    \end{tabular}
    \end{center}

\caption{\footnotesize \textbf{LeadingOnes-TrailingZeroes (LOTZ) (\(n=8\))} (a) Two individual objectives (LeadingOnes and TrailingZeroes) of the LOTZ problem with respect to the number of ones. (b) Objective space, where the number associated with a solution means how many solutions in the decision space map to that solution. (c) The level of solutions with respect to the number of ones based on the Pareto non-dominated sorting. 
In (b) and (c), green, red and blue points indicate Pareto optimal solutions, local optimal solutions, and other solutions respectively.
}
\label{fig:LOTZ}

\end{figure*}

\begin{figure*}[htbp]
\begin{center} 
\small
    \includegraphics[scale=0.52]{figures/legend_horizontal.png}
    \renewcommand{\arraystretch}{0.1}
    \begin{tabular}{@{}c@{}c@{}c}
		\includegraphics[scale=0.45]{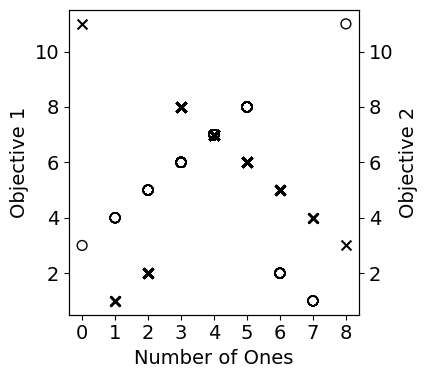}~&~
		\includegraphics[scale=0.45]{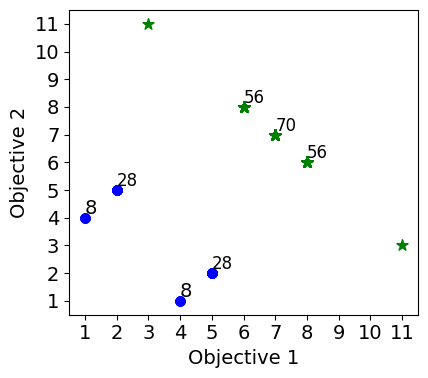}~&~
        \includegraphics[scale=0.45]{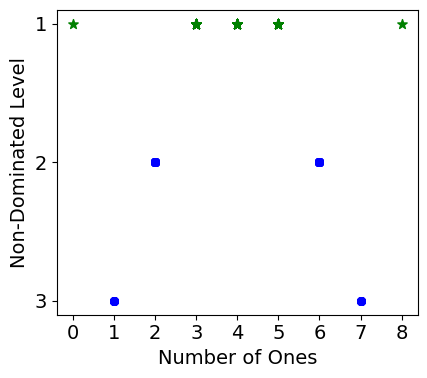}\\
		(a) Two individual objectives wrt \#\(1\) ~&~ 
		(b) Objective space ~&~
        (c) Nondominated level wrt \#\(1\) 
    \end{tabular}
    \end{center}

\caption{\footnotesize \textbf{OneJump-ZeroJump (OJZJ)} (\(n=8\), \(k=2\) where \(k\) is the jump parameter) (a) Two individual objectives (OneJump and ZeroJump) of the OJZJ problem with respect to the number of ones. (b) Objective space, where the number associated with a solution means how many solutions in the decision space map to that solution. (c) The level of solutions with respect to the number of ones based on the Pareto non-dominated sorting. 
In (b) and (c), green, red and blue points indicate Pareto optimal solutions, local optimal solutions, and other solutions respectively. 
} 
\label{fig:OJZJ}

\end{figure*}

\begin{figure*}[htbp]
\begin{center} 
\small
    \includegraphics[scale=0.52]{figures/legend_horizontal.png}
    \renewcommand{\arraystretch}{0.1}
    \begin{tabular}{@{}c@{}c@{}c}
		\includegraphics[scale=0.45]{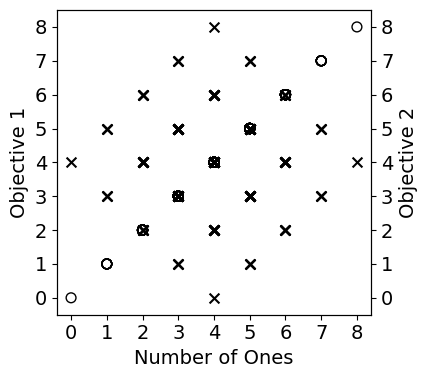}~&~
		\includegraphics[scale=0.45]{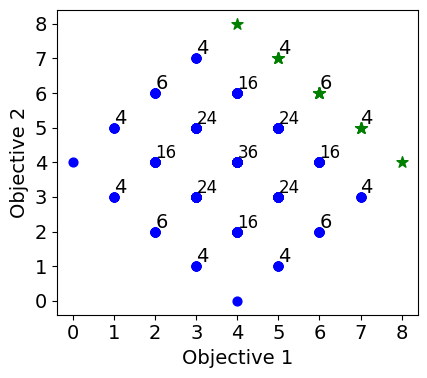}~&~
        \includegraphics[scale=0.45]{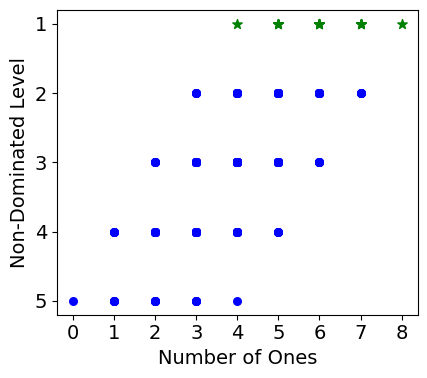}\\
		(a) Two individual objectives wrt \#\(1\) ~&~ 
		(b) Objective space ~&~
        (c) Nondominated level wrt \#\(1\) 
    \end{tabular}
    \end{center}

\caption{\footnotesize \textbf{CountOnes-CountZeroes (COCZ) (\(n=8\))} (a) Two individual objectives (OneMax and CountOnes) of the COCZ problem with respect to the number of ones. (b) Objective space, where the number associated with a solution means how many solutions in the decision space map to that solution. (c) The level of solutions with respect to the number of ones based on the Pareto non-dominated sorting. 
In (b) and (c), green, red and blue points indicate Pareto optimal solutions, local optimal solutions, and other solutions respectively.
}

\label{fig:COCZ}
\end{figure*}

\subsection{OneJump-ZeroJump (OJZJ)}\label{sec:OJZJ}
OJZJ~\cite{doerr_theoretical_2021-1} is to simultaneously maximise the number of ones and zeros while introducing a deceptive fitness gap that requires crossing a valley in each objective. The size of the valley is controlled by a parameter called the \textit{jump parameter} \(k\).
Its single-objective version, \textit{OneJump} function (Eq.~\ref{eq:jump}), is known for the presence of a local optimum. Interestingly, in the multi-objective version, the local optimum becomes a Pareto optimal solution. 
Figure~\ref{fig:OJZJ}(a) plots the solutions on two individual objectives (OneJump and ZeroJump) with respect to the number of ones for the 8-bit OJZJ problem with $k=2$. As seen, on the objective $f_1$, the solution with five ones is a local optimum for that objective; however, it becomes non-dominated if considering the objective $f_2$. The global optimal solution on each objective (i.e., $(1^n)$ and $(0^n)$) becomes the boundary solution in this two-objective problem.    

A difference of OJZJ from the previous problems is that it has disjoint optimal solutions. 
As can be seen in Figure~\ref{fig:OJZJ}(c), 
there is a gap between the boundary solutions and the remaining Pareto optimal solutions. 
This brings difficulty in finding the whole Pareto front, especially with a big $k$.
The ratio of Pareto optimal solutions in the problem depends on the size of the valley (i.e., $k$) as only solutions in the valley are dominated ones. When $k$ is small, relative to the number of total bits ($n$), the ratio is high. For example, if \( k\leq\frac{n}{2}-\sqrt{\frac{n\ln4}{2}}\), the ratio is not less than 0.5 with a sufficiently large \(n\) (Prop.~\ref{prop:OJZJ-k}). When $k$ is large (e.g., \(k=\frac{n}{2}-1\)), the ratio converges toward 0 with a sufficiently large \(n\) (Prop.~\ref{prop:OJZJ-r}).

\subsection{CountOnes-CountZeroes (COCZ)}



COCZ~\cite{laumanns_running_2004} 
is a problem known for having asymmetric objectives. 
The problem consists of two parts: a cooperative part (the first half of the decision variables) and a conflicting part (the second half of the decision variables).
In the cooperative part, both objectives aim to maximise the number of ones (like two OneMax functions).
In the conflicting part, one objective maximises the number of ones and another objective maximises the number of zeroes (like OneMinMax).
Formally, COCZ can be formulated below.
\begin{definition}\label{def:COCZ}(\textit{CountOne-CountZeroes (COCZ)}~\cite{laumanns_running_2004})
Let \( x \in \{0,1\}^n \) be a bit-string of even bits. The COCZ function is defined as
\begin{align}
f_1(x) = \sum_{i=1}^{n} x_i \quad \text{and} \quad f_2(x) = \sum_{i=1}^{n/2} x_i + \sum_{i=n/2+1}^{n} (1 - x_i).
\end{align}
\end{definition}

The first objective is the OneMax function, and the second objective is called the single-objective \textit{CountOnes} function~\cite{droste_rigorous_1998} which treats half of the decision variables as OneMax and another half as OneMin.
This feature makes the problem non-symmetric and non-completely conflicting with respect to the two objectives. 

Figure~\ref{fig:COCZ} illustrates the 8-bit COCZ function. As can be seen from Figure~\ref{fig:COCZ}(b), the solutions in the objective space form a diamond shape. This occurrence can be attributed to the fact that the cooperative part of the bit-string determines the distance between a solution and the Pareto front (from the upper right to the lower left), the conflicting part alone controls the positions of solutions on the Pareto front. This bipartite feature partly resembles several continuous benchmark suites widely used in the empirical community, such as ZDT~\cite{Zitzler2000} and WFG~\cite{Deb2005a}.
In addition, the ratio of the Pareto optimal solutions is low (\(1/2^{n/2}\)) as only the conflicting part affects the Pareto front.

\begin{figure*}[htbp]
\begin{center} 
\small
    \includegraphics[scale=0.52]{figures/legend_horizontal.png}
    \renewcommand{\arraystretch}{0.1}
    \begin{tabular}{@{}c@{}c@{}c}
		\includegraphics[scale=0.45]{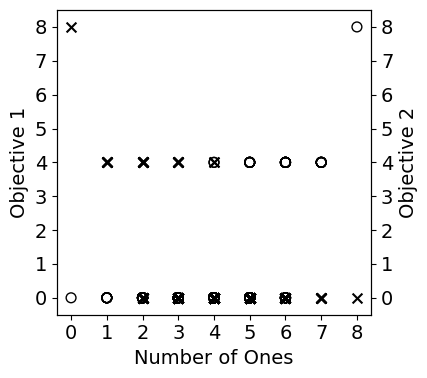}~&~
		\includegraphics[scale=0.45]{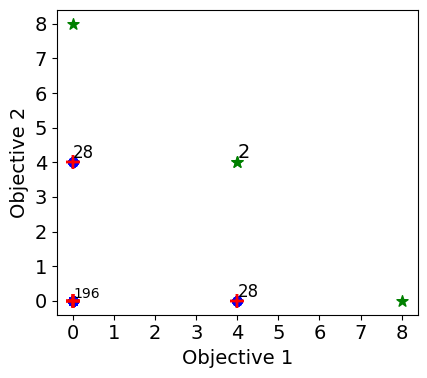}~&~
        \includegraphics[scale=0.45]{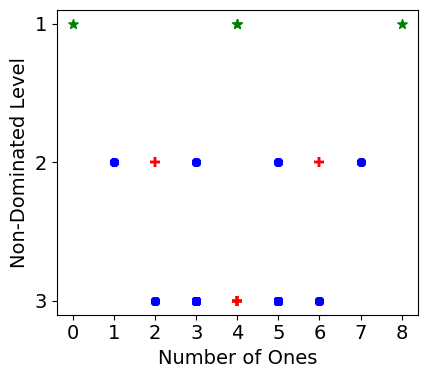}\\
		(a) Two individual objectives wrt \#\(1\) ~&~ 
		(b) Objective space ~&~
        (c) Nondominated level wrt \#\(1\) 
    \end{tabular}
    \end{center}
\caption{\footnotesize \textbf{OneRoyalRoad-ZeroRoyalRoad (ORZR)} (\(n=8\), \(b=2\) where \(b\) is the number of blocks) (a) Two individual objectives (OneRoyalRoad and ZeroRoyalRoad) of the ORZR problem with respect to the number of ones. (b) Objective space, where the number associated with a solution means how many solutions in the decision space map to that solution. (c) The level of solutions with respect to the number of ones based on the Pareto non-dominated sorting. 
In (b) and (c), green, red and blue points indicate Pareto optimal solutions, local optimal solutions, and other solutions respectively. 
In this problem, local optimal solutions (red crosses) and other solutions (blue dots) overlap in (b).}
\label{fig:ORZR}
\end{figure*}

\subsection{OneRoyalRoad-ZeroRoyalRoad (ORZR)}\label{sec:ORZR}
Unlike the previous problems, to our knowledge ORZR has not yet been considered in runtime analysis. It divides the bit-string into equal-sized blocks, each requiring all ones or zeroes to contribute to the objective (see Section~\ref{sec:royalroad}).
This design leads to disjoint optimal solutions as they are at least \(b\) bits from each other (where \(b\) denotes the number of bits in a block), and a low ratio of Pareto optimal solutions (i.e., \(\frac{2^{n/b}}{2^n}\)) as optimal solutions must consist of only blocks of all ones (or all zeroes). For the formal form of Pareto set and Pareto local optimal solutions, see Prop.~\ref{prop:ORZR-ps} and Prop.~\ref{prop:ORZR-plo} in the Appendix.

The single-objective RoyalRoad function~\cite{mitchell_royal_1992} is known for the presence of the neutral area, where there exists at least a solution that all of its neighbours are of the same fitness. 
Such a neutral area is considered a challenging feature~\cite{van_nimwegen_statistical_1999, verel_local_2011}. The multi-objective version inherits this feature; there exist solutions such that all of their neighbours are non-dominated to itself.
This introduces a feature that is prevalent in practical problems but does not exist in the other existing multi-objective pseudo-Boolean functions -- Pareto local optimality. 
Figure~\ref{fig:ORZR} illustrates the 8-bit ORZR with two blocks.
As can be seen from Figure~\ref{fig:ORZR}(c), there are three local optimal solutions whose neighbours are not better than them.
It is worth noting that in ORZR the local optimal solutions and their neighbours are not ``incomparable'' but have the same objective values (i.e., mapped into the same objective vectors, see Figure~\ref{fig:ORZR}(b)). This may pose the challenge for mainstream MOEAs that maintain diversity only in the objective space, as they cannot identify potential solutions that help in finding promising areas \cite{ren2024maintaining}.

In addition, ORZR exhibits a property that is common in real-world applications but rare in the pseudo-Boolean benchmarks. While optimal solutions require each block to consist entirely of ones or zeroes, some blocks in random solutions have roughly the same number of ones and zeroes.
This creates a position at the ``centre'' of a plateau where improving the block requires multiple simultaneous bit flips. This makes ORZR challenging for algorithms like SEMO that relies on local variation.

\section{Mix-and-Match Multi-Objective Boolean Functions}\label{sec:mix-and-match}

In the previous section, the bi-objective functions are formed by extending the symmetric version of single-objective ones. In this section, we introduce the asymmetric mixes of two different objectives: 
the first objective remains its original form, while the second is its symmetric counterpart. For example, mixing LeadingOnes and OneJump yields \textit{LeadingOnes-ZeroJump}.
Due to space limitation, the mixes involving OneMax are provided separately in the appendix, as they often trivialise the problem. Formal properties (e.g., the Pareto solution set and local optimal set) of all the new benchmarks are given Prop.~\ref{prop:ORZR-ps} to Prop.~\ref{prop:OMZR-ps} in the Appendix.

\begin{figure*}[htbp]
\begin{center} 
\small
    \includegraphics[scale=0.52]{figures/legend_horizontal.png}
    \renewcommand{\arraystretch}{0.1}
    \begin{tabular}{@{}c@{}c@{}c}
		\includegraphics[scale=0.45]{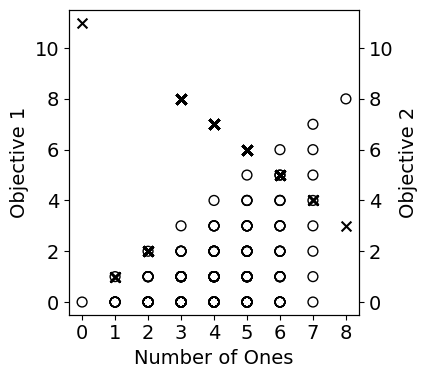}~&~
		\includegraphics[scale=0.45]{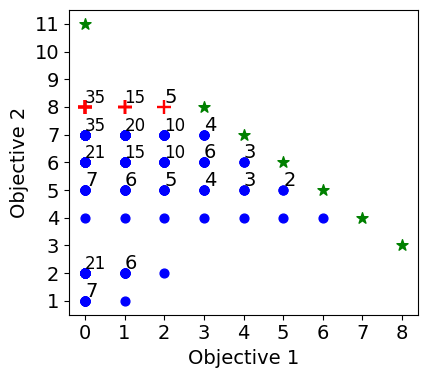}~&~
        \includegraphics[scale=0.45]{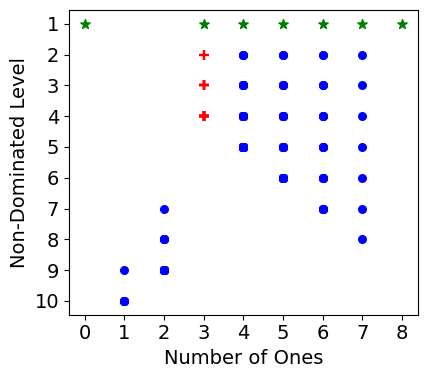}\\
		(a) Two individual objectives wrt \#\(1\) ~&~ 
		(b) Objective space ~&~
        (c) Nondominated level wrt \#\(1\)  
    \end{tabular}
    \end{center}
\caption{\footnotesize \textbf{LeadingOnes-ZeroJump (LOZJ)} (\(n=8\), \(k=3\) where $k$ is the jump parameter, controlling the size of the valley). (a) Two individual objectives (LeadingOnes and ZeroJump) of the LOZJ problem with respect to the number of ones. (b) Objective space, where the number associated with a solution means how many solutions in the decision space map to that solution. (c) The level of solutions with respect to the number of ones based on the Pareto non-dominated sorting. 
In (b) and (c), green, red and blue points indicate Pareto optimal solutions, local optimal solutions and other solutions, respectively.
}
\label{fig:LOZJ}
\end{figure*}

\subsection{LeadingOnes-ZeroJump (LOZJ)}\label{sec:LOZJ}

The problem LOZJ combines the LeadingOnes function with the ZeroJump function, and it simultaneously maximises the number of consecutive ones (from left to the right), and the number of zeroes in the bit-string with a valley (determined by the parameter $k$). 
Formally, LOZJ can be formulated as follows.

\begin{definition} (\textit{LeadingOnes-ZeroJump}).
Let \( x \in \{0,1\}^n \) be a bit-string of length \( n \) and \(k\) be the jump parameter with \(1<k<\frac{n}{2}\). The problem is defined as:
\begin{align}
\begin{split}
f_1(x) = \sum_{i=1}^{n} \prod_{j=1}^{i} x_j, \,
f_2(x) = 
\begin{cases}
k + |x|_0, & \text{if }  |x|_0 \leq n-k\,or\,x=0^n \\
n - |x|_0 , & \text{otherwise}.
\end{cases}
\end{split}
\end{align}
\end{definition}

As LOZJ combines LeadingOnes-TrailingZeroes with OneJump-ZeroJump, it inherits their characteristics such as having disjoint optimal solutions and a low ratio of optimal solutions (Table~\ref{tab:features}). However, a distinct characteristic of LOZJ is its Pareto local optimality.
Figure~\ref{fig:LOZJ}(c) shows three local optima near the valley (in red), corresponding to ZeroJump's local optima. For these local optimal solutions, flipping any bit reduces the ZeroJump value, hence these solutions are non-dominated by their neighbours. Take the local optimal solution 10010010 in Figure~\ref{fig:LOZJ}(c) as an example, where the solution is located in non-dominated level 3 with objective values $(1,7)$. Flipping any of its bits from one to zero will potentially lead to one of the three solutions, 00010010 (level 9), 10000010 and 10010000 (level 8). 
Likewise, flipping any bit from zero to one will potentially lead to one of the five solutions, 11010010 (level 3), 10110010, 10011010, 10010110 and 10010011 (level 4).

It is worth mentioning is that the valley in LOZJ poses a bigger challenge for MOEAs to reach the all-zero optimal solution than in OneJump-ZeroJump (OJZJ).
To do so, (i.e., the top-left solution in Figure~\ref{fig:LOZJ}(b)), MOEAs generally need to find the nearest Pareto optimal solution first, i.e., the one with objective values \((3,7)\) in Figure~\ref{fig:LOZJ}(b).
However, in LOZJ, this solution corresponds to only one bit-string (11100000), whereas in OJZJ, it corresponds to any bit-string with exactly three \(1\) bits (e.g., 56 bit-strings in Figure~\ref{fig:OJZJ}(b)).
This many-to-one mapping in OJZJ enables crossover to speed up the jump~\cite{doerr_runtime_2023}, but this is not possible in LOZJ.
Such hard-to-reach boundary solutions also appear in real-world problems~\cite{wang2018generator, chu2024improving}.

\begin{figure*}[htbp]
\begin{center} 
\small
    \includegraphics[scale=0.52]{figures/legend_horizontal.png}
    \renewcommand{\arraystretch}{0.1}
    \begin{tabular}{@{}c@{}c@{}c}
		\includegraphics[scale=0.45]{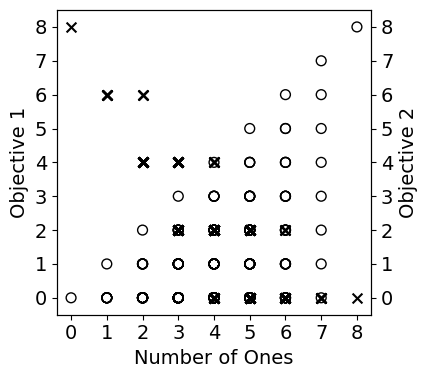}~&~
		\includegraphics[scale=0.45]{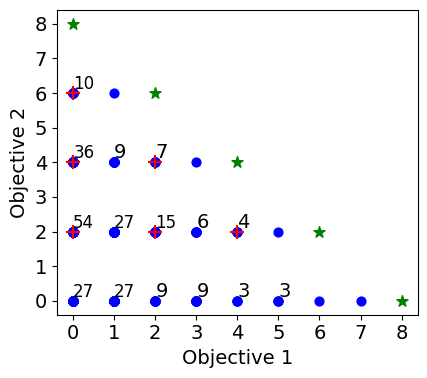}~&~
        \includegraphics[scale=0.45]{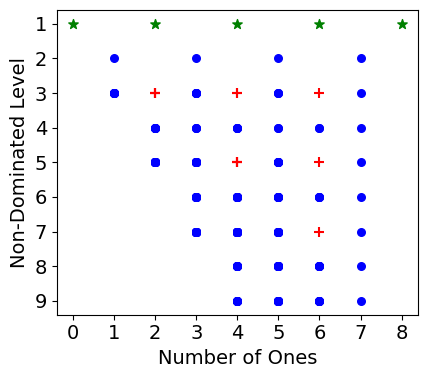}\\
		(a) Two individual objectives wrt \#\(1\) ~&~ 
		(b) Objective space ~&~
        (c) Nondominated level wrt \#\(1\) 
    \end{tabular}
    \end{center}
\caption{\footnotesize \textbf{LeadingOnes-ZeroRoyalRoad (LOZR)} (\(n=8\), \(b=4\) where \(b\) denotes the number of blocks) (a) Two individual objectives (LeadingOnes and ZeroRoyalRoad) of the LOZR problem with respect to the number of ones. (b) Objective space, where the number associated with a solution means how many solutions in the decision space map to that solution. (c) The level of solutions with respect to the number of ones based on the Pareto non-dominated sorting. 
In (b) and (c), green, red and blue points indicate Pareto optimal solutions, local optimal solutions, and other solutions respectively. 
In this problem, the local optimal solutions (red cross) overlap with other solutions (blue dots) in the (b).
}
\label{fig:LOZR}
\end{figure*}


\subsection{LeadingOnes-ZeroRoyalRoad (LOZR)}

LOZR combines the LeadingOnes function with the ZeroRoyalRoad function, simultaneously maximising the number of consecutive ones (from left to right) in a bit-string, and the number of blocks with all zero bits.
Formally, LOZR can be formulated as follows.

\begin{definition} (\textit{LeadingOnes-ZeroRoyalRoad}).
Let \( x \in \{0,1\}^n \) be a bit-string of length \( n \), partitioned into \( b \) disjoint blocks \( S_1, S_2, \dots, S_b \), with each of length \( \ell \) (where \( n = b \cdot \ell , b>1\)). 
\begin{align}
\begin{split}
f_1(x) = \sum_{i=1}^{n} \prod_{j=1}^{i} x_j,  \quad \text{and} \quad f_2(x) = \sum_{j=1}^{b} \bigl(\ell\prod_{i \in S_j} (1 - x_i)\bigr).
\end{split}
\end{align}
\end{definition}

LOZR shares many characteristics with OneRoyalRoad-Zero-\\RoyalRoad (ORZR), for example, having disjoint optimal solutions, a low ratio of optimal solutions, and Pareto local optimality (Table~\ref{tab:features}). However, an interesting difference is that in ORZR, Pareto local optimal solutions only exist when the block length is greater than three (\(\ell>3\)), whereas in LOZR they exist with any legitimate block length (i.e., \(\ell \geq2\)). The reason for this is as follows.

In ORZR, a Pareto local optimal solution needs to contain at least one block that is neither \(1^\ell\) nor \(0^\ell\), preventing it from being globally optimal.
For such a block, flipping a single bit must not turn it into \(1^\ell\) or \(0^\ell\), as that would improve at least one objective (while the other stays unchanged).
This condition can only be met when the block length \(\ell>3\), i.e., having at least two \(1\)s and at least two \(0\)s.
In contrast, in LOZR a requirement for being a Pareto local optimal solution is that a single-bit flip cannot turn a block into \(0^\ell\) (but can be \(1^\ell\) as it only considers the ZeroRoyalRoad function).
As a result, LOZR does not require a block length greater than three because a block of length two (i.e., 11) is sufficient to prevent itself from being turned into an all-zero block via a single-bit flip.

\begin{figure*}[htbp]
\begin{center} 
\small
    \includegraphics[scale=0.52]{figures/legend_horizontal.png}
    \renewcommand{\arraystretch}{0.1}
    \begin{tabular}{@{}c@{}c@{}c}
		\includegraphics[scale=0.45]{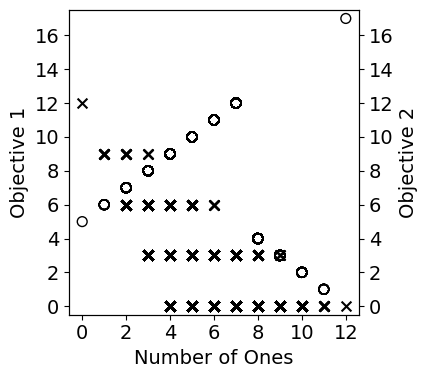}~&~
		\includegraphics[scale=0.45]{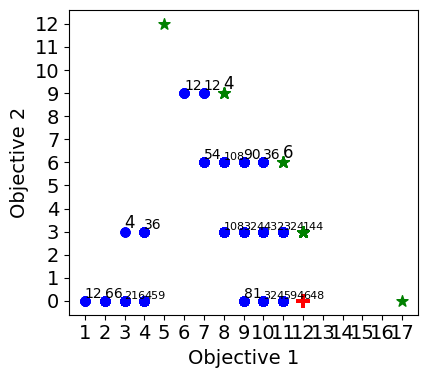}~&~
        \includegraphics[scale=0.45]{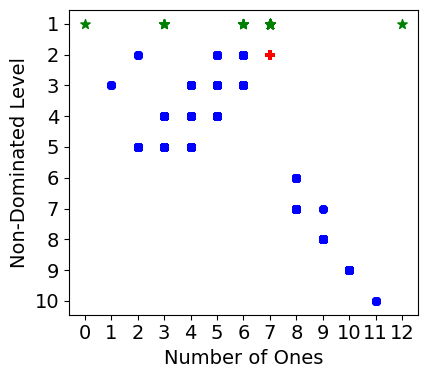}\\
		(a) Two individual objectives wrt \#\(1\) ~&~ 
		(b) Objective space ~&~
        (c) Nondominated level wrt \#\(1\) 
    \end{tabular}
    \end{center}
\caption{\footnotesize \textbf{OneJump-ZeroRoyalRoad (OJZR)} (\(n=12, k=5, b=4\) where \(k\) is the jump parameter and \(b\) is the number of blocks). 
(a) Two individual objectives of OJZR with respect to the number of ones. (b) Objective space, where the number associated with a solution means how many solutions in the decision space map to that solution. (c) The non-dominated level of solutions with respect to the number of ones. In (b) and (c), green, red and blue points indicate Pareto optimal solutions, local optimal solutions, and other solutions respectively. 
In this problem, the local optimal solution (red cross) overlaps with other solutions (blue dots) in (b). Note that the Pareto front in (b) is not linear --- the Pareto optimal solution \((11,3)\) forms a concave region. 
}
\label{fig:OJZR}
\end{figure*}

\begin{figure}[htbp]
\begin{center} 
\small
    \includegraphics[scale=0.37]{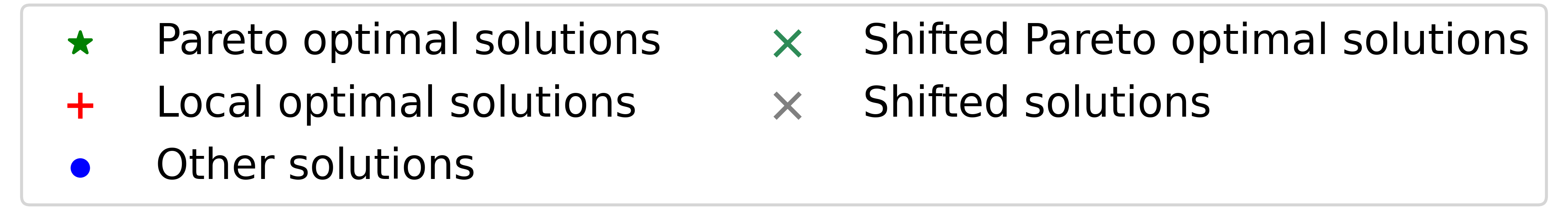}
    \renewcommand{\arraystretch}{0.1}
    \begin{tabular}{@{}c@{  }@{}c}
		\includegraphics[scale=0.36]{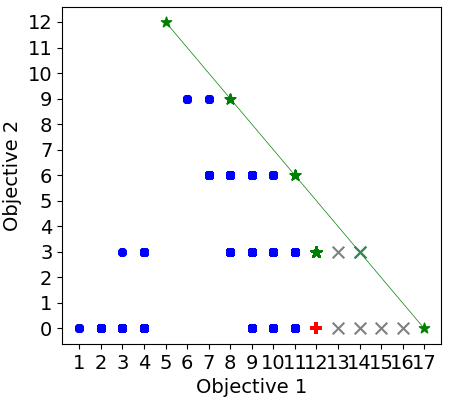} ~&~
		\includegraphics[scale=0.36]{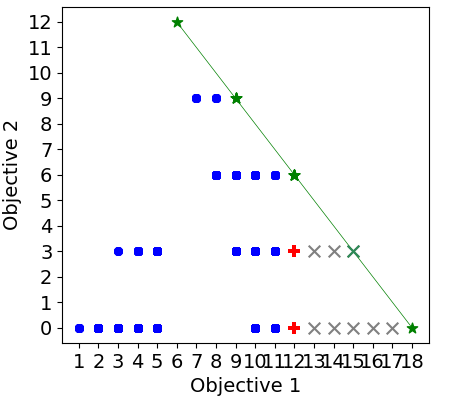}  \\
		(a) With a concave region (\(k=4\)) ~&~ 
		(b) Without a concave region (\(k=5\))
    \end{tabular}
    \end{center}
\caption{\footnotesize \textbf{The objective space of two OneJump-ZeroRoyalRoad (OJZR) instances} (\(n=12, b=4, k=5 \textrm{ or } k=6\)): one having a concave Pareto front and the other having a linear Pareto front, where \(b\) is the number of blocks and \(k\) is the jump parameter. (a) The OJZR instance with \(k=5\) 
(b) The OJZR instance with \(k=6\). In both cases, by OneJump, several solutions (the crosses) are shifted to the left. When \(k=5\), the solution \((11,3)\) becomes a Pareto optimal solution, creating a concave region in the Pareto front. In contrast, when \(k=6\) the solution \((11,3)\) is still a dominated solution since it is dominated by \((11,6)\).} 
\label{fig:OJZRexample}
\end{figure}

\subsection{OneJump-ZeroRoyalRoad (OJZR)}

OJZR combines the OneJump function with the ZeroRoyalRoad function, simultaneously maximising the number of ones with a valley (determined by the parameter \(k\)), and the number of blocks with all zero bits.
It can be formulated as follows.

\begin{definition} (\textit{OneJump-ZeroRoyalRoad}).
Let \( x \in \{0,1\}^n \) be a bit-string of length \( n \), partitioned into \( b \) disjoint blocks \( S_1, S_2, \dots, S_b \), each of length \( \ell \) (where \( n = b \cdot \ell , b>1\)). Let \( k \in \mathbb{Z}^+\) be the jump parameter with \(1<k<\lfloor \frac{n}{2} \rfloor\). The problem is defined as:
\begin{align}
\begin{split}
f_1(x) = 
\begin{cases}
k + |x|_1, &\text{if }  |x|_1 \leq n-k\, \\ & \text{ }\text{ }\text{ or}\,x=1^n \\
n - |x|_1 , &\text{otherwise}
\end{cases} \,
f_2(x) = \sum_{j=1}^{b} \bigl(\ell\prod_{i \in S_j} (1 - x_i)\bigr).
\end{split}
\end{align}
\end{definition}

OJZR exhibits all the characteristics listed in Table~\ref{tab:features}, such as the Pareto local optimality and a low ratio of Pareto optimal solutions.
Notably, it is the only problem that has a non-linear Pareto front (having a concave region). Taking the 12-bit OJZR problem (\(k=5, b=4\)) in Figure~\ref{fig:OJZR} as an example, the problem has a concave region near the solution with objective values \((11,3)\) (Figure~\ref{fig:OJZR}(b)).
The reason for this occurrence is as follows.

The occurrence of the concave region can be attributed to the fact that the OneJump objective shifts the positions of some solutions in the objective space and ``creates'' a new Pareto optimal solution. 
Figure~\ref{fig:OJZRexample}(a) illustrates and explains this in the objective space of a 12-bit OJZR problem,
where the grey and green crosses (columns $12,13,14,15$ with respect to the first objective) represent the solutions that are shifted to the left and leave a gap in the objective space (as seen in Figure~\ref{fig:OJZR}(b)).
Since the Pareto optimal solution \((13,3)\) is shifted to \((3,3)\), the solution \((11,3)\) becomes a new Pareto optimal solution.
This leads to a concave region below the original straight-line Pareto front (the green line in Figure~\ref{fig:OJZRexample}(a)).

However, such a shift does not always create a new Pareto optimal solution, depending on the interplay between the parameters $n, k, \ell$ (where \(\ell\) is the block length). When \((n-k)\mod \ell = 0\), the concave region no longer appears.
As can be seen in Figure~\ref{fig:OJZRexample}(b), the solution next to the valley \((11,3)\) cannot become a new Pareto optimal solution, as it is dominated by another Pareto optimal solution \((11,6)\).
As such, the shape of the Pareto front remains a straight line, though there is a gap between the bottom-right Pareto optimal solution and the others optimal ones. 

\section{Discussion and Limitations}


While this work focuses on pseudo-Boolean problems, other optimisation types exist (e.g., permutation and integer problems). 
Theoreticians also investigate many other multi-objective problems, such as minimum spanning tree~\cite{neumann_expected_2007, do2023rigorous,neumann_runtime_2024}, vertex cover problem~\cite{kratsch_fixed-parameter_2013}, and integer-based optimisation~\cite{rudolph_runtime_2023,doerr_runtime_2024}. 



In this study, we also limit our scope to bi-objective problems. However, many practical problems have three or more objectives. Adding more objectives may rapidly increase the complexity of optimisation problems \cite{allmendinger2022if}, and MOEAs may behave very differently \cite{Wagner2007,Ishibuchi2008a,Li2018a,li2018evolutionary}. Studies in the theory area mainly consider many-objectives problems through extending existing bi-objective problems by partitioning a bit-string into multiple sections, with each section corresponding to two objectives, such as in~\cite{huang_runtime_2021, zheng_runtime_2024, wietheger_near-tight_2024, wietheger_mathematical_2024, zheng_crowding_2024, dang_level-based_2024, opris_runtime_2024}. One feature of such problems is that the number of objectives to be optimised needs to be even. With the mix-and-match approach in this paper, one may naturally consider objectives with more than two objectives, for example, an interesting (but maybe hard) problem can be constructed by considering the three individual functions LeadingOnes, ZeroJump and ZeroRoyalRoad. 


A notable feature of all the problems considered in this paper is the common presence of many-to-one mappings between the decision space and objective space, i.e., multiple bit-strings having identical objective values. This is due to the use of counting-based (e.g., OneMax, Jump) and block-based (e.g., RoyalRoad) functions.
However, there exist many practical problems that only have one-to-one mappings between the two spaces, in particular many continuous problems like the ZDT~\cite{Zitzler2000} and DTLZ \cite{Deb2005a} suites. 
The feature of many-to-one mapping may disadvantage MOEAs that only consider the diversity of solutions in objective space, which unfortunately all mainstream MOEAs do. As presented recently, for such problems, considering the diversity of solutions in decision space can substantially speed up the search \cite{ren2024maintaining}.


Some problems constructed in this paper have local optimal solutions. This is in line with practical scenarios that many multi-objective combinatorial problems have local optimal solutions \cite{paquete_local_2007}. In general, there are two types of Pareto local optimal solutions, plateau-type and non-plateau-type. 
Plateau-type solutions have at least one identical neighbouring solution. In our problems, such plateau local optimal solutions, introduced by the RoyalRoad function, scatter dispersedly in the search space (e.g., in LeadingOnes-ZeroRoyalRoad). This is aligned with some practical optimisation scenarios like time tabling~\cite{sakal_genotype_2023}. In contrast, the non-plateau local optimal solutions, primarily caused by the Jump function, concentrate in specific regions of the search space, as seen in the problems LeadingOnes-ZeroJump and OneJump-ZeroRoyalRoad. However, in some practical problems, the non-plateau local optimal solutions can be more dispersed such as in multi-objective NK-Landscape \cite{aguirre_insights_2004,liefooghe2023pareto}, making MOEAs being trapped in different regions~\cite{Li2023}.

\section{Conclusion}
In this paper, we conducted a short review of pseudo-Boolean problems used in runtime analysis for evolutionary multi-objective optimisation. Based on characteristics that real-world optimisation problems have, we discussed commonly used benchmarks in the area, including their limitations and implications for practical use.  

We also presented several new pseudo-Boolean problems by mix-and-matching different single-objective functions, and shown that they have more characteristics commonly seen in real-world applications (see Table~\ref{tab:features}). This includes plateau-type local optima (OneJump-ZeroRoyalRoad), a non-linear Pareto front (OneJump-ZeroRoyalRoad), and hard-to-reach boundary solutions (LeadingOnes-ZeroJump) which may not be easily obtained by crossover. 
Of course, these functions may not fully reflect the complexity of practical problems, but we hope they can help strengthen the connection between theoretic and practical research in the area of evolutionary multi-objective optimisation.

\section*{Acknowledgment}
The authors would like to acknowledge the discussions with Prof Chao Qian, Dr Chao Bian, Mr Shengjie Ren, Dr Duc-Cuong Dang, and Prof Per Kristian Lehre that motivated this work.

\begingroup
\footnotesize
\setstretch{0.9}
\bibliography{reference}
\bibliographystyle{unsrt}

\endgroup

\appendix

\section*{Mix-and-Match Multi-Objective Boolean Functions with OneMax}

In the main paper, we only present the mixes of LeadingOnes, Jump and RoyalRoad as mixing OneMax with other functions may be trivial, as it simplifies the problem. 
Here, we present the mixes involving the OneMax function.  

\begin{figure*}[htbp]
\begin{center} 
    \includegraphics[scale=0.6]{figures/legend_horizontal.png}
    \begin{tabular}{@{}c@{}c@{}c}
		\includegraphics[scale=0.47]{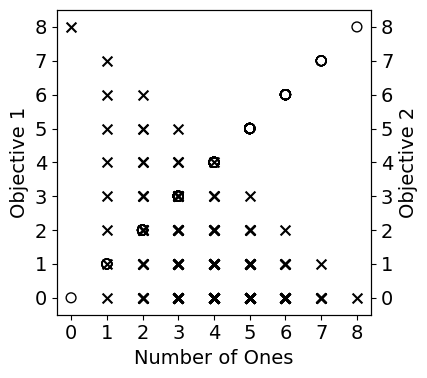}~&~
		\includegraphics[scale=0.47]{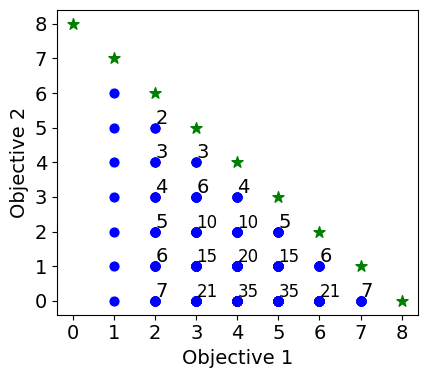}~&~
        \includegraphics[scale=0.47]{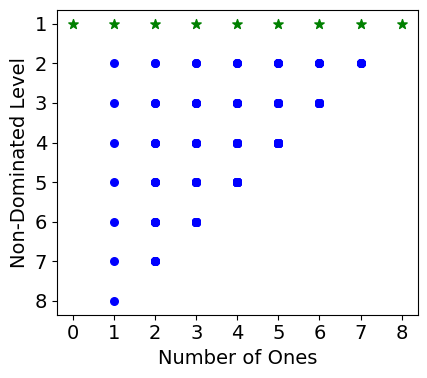}\\
		(a) Two individual objectives wrt \#\(1\) ~&~ 
		(b) Objective space ~&~
        (c) Nondominated level wrt \#\(1\) 
    \end{tabular}
    \end{center}
\caption{\footnotesize \textbf{OneMax-TrailingZeroes (OMTZ) (\(n=8\))} (a) Two individual objectives (OneMax and TrailingZeroes) of the OMTZ problem with respect to the number of ones. (b) Objective space, where the number associated with a solution means how many solutions in the decision space map to that solution. (c) The level of solutions with respect to the number of ones based on the Pareto non-dominated sorting. 
In (b) and (c), green, red and blue points indicate Pareto optimal solutions, local optimal solutions, and other solutions respectively.}
\label{fig:OMTZ}
\end{figure*}


\subsection{OneMax-TrailingZeroes (OMTZ)}
OMTZ combines the OneMax function with the TrailingZeroes function, thus a simultaneous maximisation of the number of ones in the bit-string and the number of consecutive zeroes (from right to left). 
Formally, OMTZ can be formulated as follows.
\begin{definition} (\textit{OneMax-TrailingZeroes}).
Let \( x \in \{0,1\}^n \) be a bit-string of length \( n \). The OMTZ problem is defined as:
\begin{align}
\begin{split}
f_1(x) = \sum_{i=1}^{n} x_i, \quad \text{and} \quad f_2(x) = \sum_{i=1}^{n} \prod_{j=i}^{n} (1 - x_j).
\end{split}
\end{align}
\end{definition}


OMTZ shares some characteristics with LeadingOnes-TrailingZeroes (LOTZ), 
such as a linear Pareto front, no disjoint optimal solutions, no local optima, and a low ratio of Pareto optimal solutions.
The primary difference between them is that OMTZ is non-symmetric with respect to its objectives, which is commonly seen in real-world problems.

The OneMax objective in OMTZ is easier to improve on than the TrailingZeroes objective, since improving on OneMax involves flipping any bit from zero to one, whereas improving on TrailingZeroes involves flipping a specific one bit to zero.
This difference creates an objective imbalance, a common feature in many test suites (e.g., ZDT \cite{Zitzler2000}).
Figure~\ref{fig:OMTZ} illustrates the 8-bit OMTZ.
As seen in Figure~\ref{fig:OMTZ}(b), more solutions concentrate in a region with fairly good value on the objective OneMax but poor value on the objective TrailingZeroes.

\begin{figure*}[htbp]
\begin{center} 
    \includegraphics[scale=0.6]{figures/legend_horizontal.png}
    \begin{tabular}{@{}c@{}c@{}c}
		\includegraphics[scale=0.47]{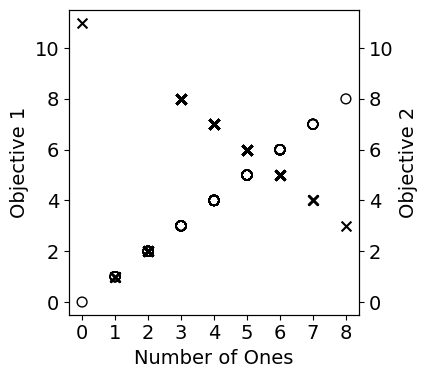}~&~
		\includegraphics[scale=0.47]{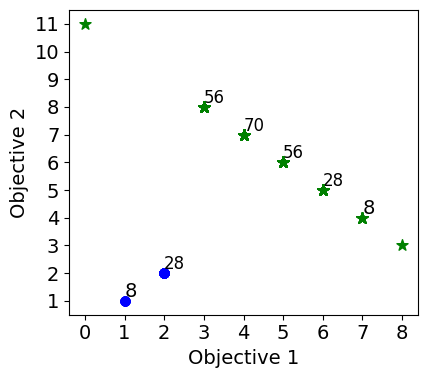}~&~
        \includegraphics[scale=0.47]{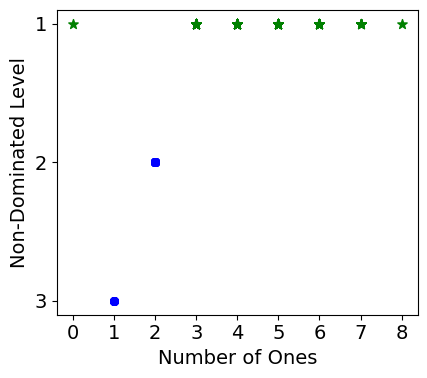}\\
		(a) Two individual objectives wrt \#\(1\) ~&~ 
		(b) Objective space ~&~
        (c) Nondominated level wrt \#\(1\) 
    \end{tabular}
    \end{center}
\caption{\footnotesize \textbf{OneMax-ZeroJump (OMZJ)} (\(n=8\), jump parameter \(k=3\)) (a) Two individual objectives (OneMax and ZeroJump) of the OMZJ problem with respect to the number of ones. (b) Objective space, where the number associated with a solution means how many solutions in the decision space map to that solution. (c) The level of solutions with respect to the number of ones based on the Pareto non-dominated sorting. 
In (b) and (c), green, red and blue points indicate Pareto optimal solutions, local optimal solutions, and other solutions respectively.}
\label{fig:OMZJ}
\end{figure*}


\subsection{OneMax-ZeroJump (OMZJ)}

OMZJ combines the OneMax function with the ZeroJump function, aiming to simultaneously maximise the number of ones in the bit-string while crossing a valley in the search space determined by the number of zeroes.
Like OneJump-ZeroJump (OJZJ), the size of this valley is controlled by the \emph{jump parameter} \(k\), with larger values of \(k\) making it more difficult to reach the boundary solution for the ZeroJump objective (i.e., \((0^n)\)).
Formally, OMZJ can be formulated as follows.

\begin{definition} (\textit{OneMax-ZeroJump}).
Let \( x \in \{0,1\}^n \) be a bit-string of length \( n \) and let \(k\) be a fixed jump parameter with \(1<k<\frac{n}{2}\). The OneMax-ZeroJump problem is defined as:
\begin{align}
\begin{split}
f_1(x) &= \sum_{i=1}^{n} x_i, \\
f_2(x) &= 
\begin{cases}
k + |x|_0, & \text{if }  |x|_0 \leq n-k\,or\,x=0^n \\
n - |x|_0 , & \text{otherwise}.
\end{cases}
\end{split}
\end{align}
\end{definition}

OMZJ shares some characteristics with OJZJ, such as having disjoint optimal solutions.
However, unlike OZJZ, OMZJ does not exhibit a low ratio of Pareto optimal solutions.
This is because OMZJ has only one valley (where dominated solutions are located) from the ZeroJump objective (in contrast to the two valleys in OZJZ), thus at least half of the solutions are Pareto optimal, regardless of the value of \(k\).
Figure~\ref{fig:OMZJ}(b) plots the objective space of the 8-bit OMZJ problem, where it is clear that most of the solutions are Pareto optimal ones.

\begin{figure*}[htbp]
\begin{center} 
    \includegraphics[scale=0.6]{figures/legend_horizontal.png}
    \begin{tabular}{@{}c@{}c@{}c}
		\includegraphics[scale=0.47]{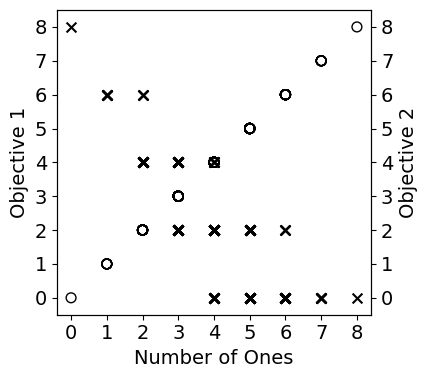}~&~
		\includegraphics[scale=0.47]{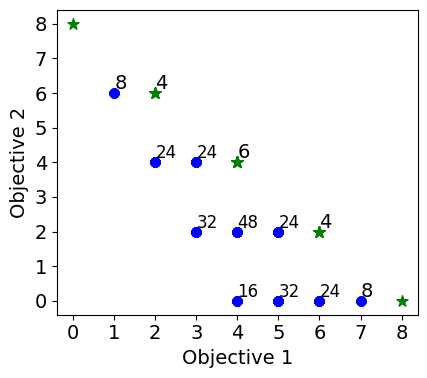}~&~
        \includegraphics[scale=0.47]{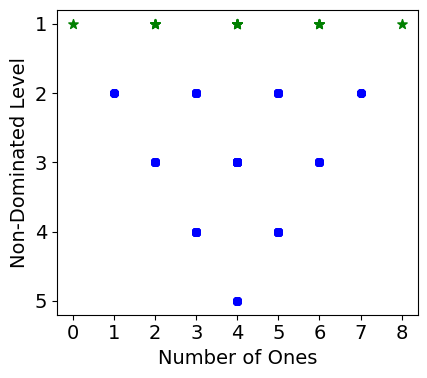}\\
		(a) Two individual objectives wrt \#\(1\) ~&~ 
		(b) Objective space ~&~
        (c) Nondominated level wrt \#\(1\) 
    \end{tabular}
    \end{center}
\caption{\footnotesize \textbf{OneMax-ZeroRoyalRoad (OMZR)} (\(n=8\), \(b=4\) where \(b\) denotes the number of blocks) (a) Two individual objectives (OneMax and ZeroRoyalRoad) of the OMZR problem with respect to the number of ones. (b) Objective space, where the number associated with a solution means how many solutions in the decision space map to that solution. (c) The level of solutions with respect to the number of ones based on the Pareto non-dominated sorting. 
In (b) and (c), green, red and blue points indicate Pareto optimal solutions, local optimal solutions and other solutions, respectively.}
\label{fig:OMZR}
\end{figure*}


\subsection{OneMax-ZeroRoyalRoad (OMZR)}
OMZR combines the OneMax function with the ZeroRoyalRoad function, and it simultaneously maximises the number of ones in a bit-string and the number of blocks with all zero bits.
Formally, OMZR can be formulated as follows.
\begin{definition} (\textit{OneMax-ZeroRoyalRoad}).
Let \( x \in \{0,1\}^n \) be a bit-string of length \( n \), partitioned into \( b \) disjoint blocks \( S_1, S_2, \dots, S_b \), each of length \( \ell \) (where \( n = b \times \ell, b>1 \)). The problem is defined as:
\begin{align}
f_1(x) = \sum_{i=1}^{n} x_i \quad \text{and} \quad
f_2(x) = \sum_{j=1}^{b} \bigl(\ell\prod_{i \in S_j} (1 - x_i)\bigr).
\end{align}
\end{definition}

Like OneRoyalRoad-ZeroRoyalRoad (ORZR), the block design of the ZeroRoyalRoad objective in the OMZR problem leads to disjoint optimal solutions and a low ratio of Pareto optimal solutions.
However, unlike ORZR, OMZR does not exhibit Pareto local optima, as any non-optimal solution can improve on its OneMax objective by flipping a zero to one.
Figure~\ref{fig:OMZR} illustrates the 8-bit OMZR problem.
As can be seen in Figure~\ref{fig:OMZR}(b), all dominated solutions can move rightward by improving on their OneMax objective (without affecting the ZeroRoyalRoad objective) until reaching a Pareto optimal solution.
This may make OMZR easier to be dealt with than the problem consisting of two RoyalRoad functions (i.e., ORZR). 

\begin{table*}[t]
  \centering
  \caption{Summary of global and local Pareto properties for the new benchmarks. \\ \(x\) denotes a solution and each solution is a bit-string of length \(n\); \(k\) denotes the jump parameter of the Jump function component; \(b\) and \(\ell\) denotes the number of blocks and the corresponding block length of the RoyalRoad function component.}
  \label{tab:benchmark-properties}
  \small
  \begin{tabular}{l@{ }p{5.5cm}p{5cm}p{4.2cm}}
    \textbf{Benchmark} & \textbf{Pareto set \(\mathcal{P}\)} & \textbf{Pareto front} & \textbf{Local optima \(\mathcal{L}\)} \\ \hline
    ORZR & 
      \(\displaystyle\{\,x \mid |\{S=S_i=1^\ell\}|+|\{S=0^\ell\}|=b\}\) 
      & \(\{(i\ell,(b-i)\ell)\mid i=0,\dots,b\}\)
      & 
      \(\displaystyle\{\,x \mid \forall S:|S|_1>1\ \wedge\ |S|_0>1,\;|\{S=1^\ell\}|+|\{S=0^\ell\}|<b\}\) (\(\ell>3\)) \\[1ex] \hline 
    LOZJ & 
      \(\{0^n\}\cup\{1^i0^{n-i}\mid i=k,\dots,n\}\) 
      & \(\{(0,n+k)\}\cup\{(i,n+k-i)\mid i=k,\dots,n\}\)
      & 
      \(\{\,1^is\mid i<k,\;|s|_0=n-k\}\) \\[1ex] \hline 
    LOZR & 
      \(\{\,1^{i\ell}0^{n-i\ell}\mid i=0,\dots,b\}\) 
      & \(\{(i\ell,(b-i)\ell)\mid i=0,\dots,b\}\)
      & 
      \(\bigl\{1^is \mid i\in\{0,\ell,\dots,(b-1)\ell\},\;S_{\lceil i/\ell\rceil+1}=0^\ell,\;\forall j>\lceil i/\ell\rceil,\;|S_j|_1>1\bigr\}\) \\[1ex] \hline 
    OJZR  & 
      \(\{1^n\}\cup\{\,x \mid |x|_1\le n-k,\;|\{1^\ell\}|+|\{0^\ell\}|=b\}\) & 
      \(\{(n+k,0)\}\cup\{(i\ell+k,n-i\ell)\mid i=0,\dots,\lfloor k/\ell\rfloor\}\) & 
      \(\{\,x\mid|x|_1=n-k,\;|\{0^\ell\}|<\lfloor k/\ell\rfloor\}\) \\[1ex]
      & \multicolumn{3}{l}{\(\!\!\)(case \((n-k)\bmod\ell=0\))} \\[1ex]
    OJZR & 
      \(\{1^n\}\cup\{x\mid|x|_1<n-k,\;|\{1^\ell\}|+|\{0^\ell\}|=b\}\,\cup\,\{x\mid|x|_1=n-k,\;|\{0^\ell\}|=\lfloor k/\ell\rfloor\}\) 
      & \(\{(n+k,0)\}\cup\{(n-k,\lfloor k/\ell\rfloor\ell)\}\cup\{(i\ell+k,n-i\ell)\mid i=0,\dots,\lfloor k/\ell\rfloor\}\)
      & same as above \\[1ex]
      & \multicolumn{3}{l}{\(\!\!\)(case \((n-k)\bmod\ell>0\))} \\[1ex] \hline 
    OMTZ &
      \(\{0^n\}\cup\{1^i0^{n-i}\mid i=0,\dots,n\}\) 
      & \(\{(i,n-i)\mid i\in\{0,1,\dots,n\}\}\)
      & none \\[1ex] \hline 
    OMZJ &
      \(\{0^n\}\cup\{x\mid|x|_0\le n-k\}\) 
      & \(\{(0, n+k))\}\cup\{(i,n+k-i)\mid i=k,\dots,n\}\)
      & none \\[1ex] \hline 
    OMZR &
      \(\{\,x\mid|\{1^\ell\}|+|\{0^\ell\}|=b\}\) 
      & \(\{(i\ell,(b-i)\ell)\mid i=0,\dots,b\}\)
      & none \\ \hline 
  \end{tabular}
\end{table*}

\section*{Formal Properties of the New Benchmarks}
In this section, we present the formal properties, such as the Pareto front and the Pareto set under general parameters, for each newly introduced benchmark.
The properties are summarised in Table~\ref{tab:benchmark-properties}. The proofs are given below.

We start from the benchmarks presented in the main paper, namely OneRoyalRoad-ZeroRoyalRoad, LeadingOnes-ZeroJump, LeadingOnes-ZeroRoyalRoad and OneJump-ZeroRoyalRoad.

\subsection{OneRoyalRoad-ZeroRoyalRoad (ORZR)}
\begin{proposition}\label{prop:ORZR-ps}
Let \(n,b,\ell\in\mathbb{N}, n=b\cdot\ell,b>1\), given a bit-string of length \(n\), partitioned into \( b \) disjoint blocks \( x = S_1  S_2  \dots  S_b \), each with the same length \( \ell \). For the benchmark ORZR, the Pareto optimal set is the following.
\begin{align}
\mathcal{P}=\{ x \,\, \bigl| \,\, |\{S \mid S=1^\ell\}|+|\{S \mid S=0^\ell\}|=b \}
\end{align}
\end{proposition}

\begin{proof}
For each \(S\) block, it either contributes to one objective only (by being \(1^\ell\) or \(0^\ell\)) by a constant quantity \(\ell\) or contributes to nothing. Denoting \(1^\ell\) or \(0^\ell\) as a completed block. This implies that all blocks must be completed to become a Pareto optimal solutions, as otherwise a solution is dominated by another solution that completes its uncompleted block.

For those solutions in \(\mathcal{P}\), all blocks are completed and we have \(f_1+f_2=n\). Therefore, solutions in \(\mathcal{P}\) are non-dominated to each other, as they cannot improve one objective without reducing another.

For \(x\notin\mathcal{P}\), they have at least one uncompleted block. Completing that block yields a solution that dominates \(x\). Therefore, \(x\notin\mathcal{P}\) are dominated.
\end{proof}

The corresponding Pareto front is thereby \(\{\bigl(i\ell, (b-i)\ell\bigr) \mid i\in\{0,1,\dots,b\}\}\).

\begin{proposition}\label{prop:ORZR-plo}
With the same notation above, given the block length \(\ell>3\), the local Pareto optimal solutions of ORZR are the following.
\begin{align}
\mathcal{L} = \big\{ x \,\big|\,\ 
  &\forall S : |S|_1 > 1 \wedge |S|_0 > 1, \\
  &|\{S  \mid s = 1^\ell\}| 
   + |\{S  \mid s = 0^\ell\}| < b
\big\}
\end{align}
\end{proposition}
\begin{proof}
We first show that the solutions \(x\in\mathcal{L}\) are non-dominated by its 1-bit neighbours. 
Since solutions in \(\mathcal{L}\) have the property \(\forall S : |S|_1 > 1 \wedge |S|_0 > 1\), all uncompleted blocks cannot be completed (become \(1^\ell\) or \(0^\ell\)) via one bit-flip, implying flipping a bit in an uncompleted block has no effect. 
On the other hand, flipping any bit in completed blocks will reduce one objective without improving another.
Thus, \(x\in\mathcal{L}\) cannot be improved via a single bit flip, meaning that it is not dominated by any of its 1-bit neighbours.

Then, we show that the solutions \(x\notin\mathcal{P}\cup\mathcal{L}\) are dominated by one of its neighbours.
Since \(\forall S  : |S|_1 > 1 \wedge |S|_0 > 1\) is false, we have \(\exists S:|S|_0\leq1 \vee |S|_1\leq1\). Since \(x\notin\mathcal{P}\), there exists at least one uncompleted block (i.e., \(|S|_1\neq 0\wedge|S|_0\neq0\)), thus we have \(\exists S:|S|_0=1 \vee |S|_1=1\). For such block, flipping that one \(1\) left (or one \(0\) left) completes this block and yields improvement on one objective, which dominates \(x\).
\end{proof}

\begin{remark}
The above proof of the Pareto local optima relies on the condition \(\ell>3\). This is because, the condition \(\forall S : |S|_1 > 1 \wedge |S|_0 > 1\) implies that the minimum \(|S|_1\) and \(|S|_0\) are both 2, thus \(\ell=|S|=|S|_1+|S|_0 \geq4\). Therefore, for any \(\ell\leq3\), the Pareto local optimal set \(\mathcal{L}\) is empty.
\end{remark}

\subsection{LeadingOnes-ZeroJump (LOZJ)}
\begin{proposition}
Let \(n\in\mathbb{N}, 1<k<\frac{n}{2}\), given a bit-string of length \(n\), for the benchmark LOZJ, the Pareto optimal set is the following.
\begin{align}
\mathcal{P}=\{\,0^n\,\}\,\cup\,\{ \,1^i0^{n-i} \mid i\in\{k,k+1,\dots,n\} \}
\end{align}
\end{proposition}
\begin{proof}
Let \(i=f_1(x)\) be the number of leading ones, and \(t=|x|_0\). Since the first \(i\) bits of \(x\) are 1s and the last \(t\) bits are 0. 
\[i+t\leq n,\]
First, we show that all solutions in \(\mathcal{P}\) are non-dominated. \(0^n\) is non-dominated, as it is the optimal solution of the ZeroJump. As for \(x_i=1^i0^{n-i}\) with \(i\geq k\), we have \(f_1(x_i)=i\) and \(f_2(x_i)=k+n-i\). Since \(f_1\) has \(i\) and \(f_2\) has \(-i\), it is obvious that no point dominates another in \(\mathcal{P}\).

Then, we show that for any \(x'\notin\mathcal{P}\), it is dominated by some \(x\in\mathcal{P}\). Denote its leading ones as \(i'\) and \(t'=|x'|_0\), we consider two cases. 

\textbf{Case 1: \(i'<k\)}. We can compare it with \(x=1^k0^{n-k}\). Then, we have\(f_1(x')=i'<k=f_1(x)\). As for \(f_2\), by \(0^{n-k}\), \(f_2(x)\) already has maximum ZeroJump value below \(0^n\), thus \(f_2(x')\leq f_2(x)\). 

\textbf{Case 2: \(i'\in\{k,\dots,n\}\)}. 
Since \(x'\) is not in \(\mathcal{P}\), it cannot have the form \(1^{i'}0^{n-i'}\), then \(i'+t'<n\). We compare it with \(x=1^{i'} 0^{n-i'}\) in \(\mathcal{P}\). Then, we have \(f_1(x')=i'=f_1(x)\) and \(f_2(x')=t'<n-i'=f_2(x)\).

In both cases, \(x\) strictly dominates \(x'\), hence no \(x' \notin\mathcal{P}\) is Pareto-optimal.
\end{proof}

The corresponding Pareto front of LOZJ is \(\{(0, n+k)\}\cup\{(k,n),(k+1,n-1),\dots,(n,k)\}\).

Notably, besides the global optima, there are Pareto local optimal solutions in LOZJ. 

\begin{proposition}
With the same notation above, the local Pareto optimal solutions of LOZJ are the following.
\begin{align}
\mathcal{L}=\{\,1^is\mid i<k, |s|_0=n-k\}
\end{align}
where \(s\) is a sub bit-string of length \(n-i\).
\end{proposition}
\begin{proof}
For any solution \(x\), let \(i\) denotes the number of leading ones and \(t=|x|_0\).

Firstly, we show that solutions in \(\mathcal{L}\) are Pareto local optimal solutions. Flipping any \(1\) among the first \(i\) bits only reduces the LeadingOnes objective.
Flipping any \(1\) to \(0\) in \(s\) only reduces the ZeroJump objective, as \(s\) has more than \(n-k\) \(0\)s and drop into the valley.
Flipping any \(0\) to \(1\) in \(s\) may or may not increase the LeadingOnes objective, but it definitely reduces the ZeroJump objective as there are fewer zeroes.
For all cases, flipping one bit cannot generate any solution that dominates solution in \(\mathcal{L}\).

Secondly, we show that any \(x\notin\mathcal{P}\cup\mathcal{L}\) is dominated by at least one of its neighbour (i.e., can be improved via a single bit flip).
Similarly, we denote \(i\) as the number of leading ones and \(t=|x|_0\).
We consider two cases.

\textbf{Case 1: \(i<k\)}. Since it is not in \(\mathcal{P}\), we have \(t\neq n-k\). If \(t<n-k\), then there is at least one \(1\) among the non-leading bits because \(t<n-k<n-i\), so flipping any such \(1\) immediately improves the ZeroJump objective. On the other hand, if \(t>n-k\), then \(f_2(x)=n-t=i\). Here, flipping any \(0\) to \(1\) can improve the ZeroJump objective (and may improve the LeadingOnes objective if that \(0\) is at position \(i+1\)). For both cases, that improved neighbouring solution dominates the current solution.

\textbf{Case 2: \(i\geq k\)}. Since \(x\notin\mathcal{P}\cup\mathcal{L}\), we have \(i+t<n\). Here, there exists at least one \(1\) that is not in the first \(i+1\) \(1\)s (otherwise \(i\) will be larger). Flipping such \(1\) to \(0\) improve the ZeroJump objective without affecting the LeadingOnes objective, thus dominating the current solution \(x\).

In all cases, \(x\) has a one-bit neighbour that dominates it, therefore not being Pareto local optimal.
\end{proof}

\subsection{LeadingOnes-ZeroRoyalRoad (LOZR)}
\begin{proposition}
Let \(n\in\mathbb{N}, 1<k<\frac{n}{2}\), for the benchmark LOZR, partitioned into \( b \) disjoint blocks \( x = S_1  S_2  \dots  S_b \), each with the same length \( \ell \). For the benchmark LOZR, the Pareto optimal set is the following.
\begin{align}
\mathcal{P}=\{ \,1^i0^{n-i} \mid i\in\{0,\ell, 2\ell,\dots,b\ell\} \}
\end{align}
\end{proposition}
\begin{proof}
Firstly, we show that solutions in \(\mathcal{P}\) are non-dominated. Let \(i=f_1(x)\) be the number of leading ones. 
Since \(n=b\ell\), the format \( \,1^i0^{n-i}\) where \(i\in\{0, \ell, 2\ell,\dots,b\ell\} \) indicates that the trailing zeroes resembles \(\frac{n-i}{\ell}\) blocks of \(0^\ell\). This yields an objective of \((i, n-i)\). Clearly, \(i\) and \(-i\) in the objective values indicate that solutions in this set are non-dominated to each other. 

Secondly, we show that \(x\notin\mathcal{P}\) are dominated.
For any solutions \(x\notin\mathcal{P}\), let \(i = f_1(x)\), we can rewrite those solutions as \(1^i 0 s\) for \(i\in\{0,\dots,n-1\}\) and arbitrary \(s\) where \(|s|=n-2\).
If \(i \mod \ell=0 \), the rest of the bits cannot be entirely \(0^{n-i}\) as \(x\notin\mathcal{P}\), indicating that there exist blocks that are not \(0^\ell\).
Clearly, \(x\) is dominated by the solution \(1^i 0^{n-i}\) as \(1^i 0^{n-i}\) has a higher ZeroRoyalRoad objective. If \(i \mod \ell\neq0\), it means that flipping the \(0\) at the \(i+1\) position (before \(s\)) will not interrupt any \(0^\ell\) blocks in \(s\) (if there is any) but only improve the LeadingOnes objective, making \(x\) dominated by this new solution.
\end{proof}

The corresponding Pareto front is thereby \(\{\bigl(i\ell, (b-i)\ell\bigr) \mid i\in\{0,1,\dots,b\}\}\), identical to the Pareto front of ORZR.

\begin{proposition}
With the same notation above, the local Pareto optimal solutions of LOZR are the following.
\begin{align}
\begin{split}
\mathcal{L} = \big\{ \,1^is \,\big|\,\ 
  &\forall S_j,j>\lceil\frac{i}{\ell}\rceil : |S|_1>1, \\
  &S_{\lceil\frac{i}{\ell}\rceil+1}=0^\ell\\
  & i\in\{0, \ell, 2\ell,\dots,(b-1)\ell\}
\big\}
\end{split}
\end{align}
where \(s\) is an arbitrary sub bit-string of length \(n-i\).
\end{proposition}
\begin{proof}
We first show that the solutions \(x\in\mathcal{L}\) are non-dominated by its 1-bit neighbours. 
The first condition \(\forall S_j,j>\lceil\frac{i}{\ell}\rceil : |S|_1>1\) only considers the blocks that do not contain any leading ones. This condition implies that the rest of the blocks cannot be turned into \(0^\ell\) via flipping an \(1\) to \(0\), thus the ZeroRoyalRoad objective cannot be improved.
The second condition implies that the first block that does not contain leading ones is an all-zeroes block, and the third condition implies that extending the leading ones will destroy this all-zeroes block, thus the LeadingOnes objective cannot be improved without reducing the ZeroRoyalRoad objective.
In both cases, flipping one bit cannot yield any solution that dominates the current one.

Then, we show that the solutions \(x\notin\mathcal{P}\cup\mathcal{L}\) are dominated by one of its neighbours.
Consider the case \(i\notin\{0, \ell, 2\ell, \dots, (b-1)n\}\), the block containing position \(i\) is not completed and extending the leading ones to \(i+1\) can improve the LeadingOnes objective without reducing the ZeroRoyalRoad objective, as it has not reached the next block yet.
As for the case \(i\notin\{0, \ell, 2\ell, \dots, (b-1)n\), since \(x\notin\mathcal{L}\), the conditions \(\forall S_j,j>\lceil\frac{i}{\ell}\rceil : |S|_1>1\) and \(S_{\lceil\frac{i}{\ell}\rceil+1}=0^\ell\)  cannot hold at the same time.
If the former does not hold, then there exists a block that does not contain any leading ones, but there is only one \(1\). Flipping this \(1\) to \(0\) yields a \(0^\ell\) block and improves the ZeroRoyalRoad objective.
If the latter does not hold, extending the leading ones does not break any \(0^\ell\) and thus it only improves the LeadingOnes objective.
For all cases, the solution \(x\) can be improved by a single bit flip and this new solution dominates \(x\).
\end{proof}

\subsection{OneJump-ZeroRoyalRoad (OJZR)}
As illustrated in the main paper, the Pareto set of OJZR can be linear or non-linear, depending on the condition \((n-k)\mod \ell=0\) where \(n\) is the length of the bit-string, \(k\) is the jump parameter of OneJump and \(ell\) is the block length of ZeroRoyalRoad.
These two cases have different Pareto set, and thus we treat them as two propositions as the following.

\begin{proposition}\label{prop:OJZR-ps1}
Let \(n\in\mathbb{N}, 1<k<\frac{n}{2}\), given a bit-string of length \(n\), partitioned into \( b \) disjoint blocks \( x = S_1  S_2  \dots  S_b \), each with the same length \( \ell<k \). For the benchmark OJZR, when \((n-k)\mod\ell=0\), the Pareto optimal set is the following.
\begin{align}
\begin{split}
&\mathcal{P}=\{ \,1^n\,\} \;\cup \\ 
&\{\; x \,\, \bigl| \,\, |x|_1\leq n-k, \; |\{S \mid S=1^\ell\}|+|\{S \mid S=0^\ell\}|=b \}
\end{split}
\end{align}
\end{proposition}
\begin{proof}
One may notice that this Pareto set is similar to the Pareto set of ORZR (Prop.~\ref{prop:ORZR-ps}). The key difference is that the OneJump function makes those Pareto solutions with \(n-k<|x|_1<n\) no longer Pareto optimal, leaving the rest Pareto optimal solutions in \(\mathcal{P}\).

For any solution in \(x\in\mathcal{P}\), breaking \(0^\ell\) blocks decreases the ZeroRoyalRoad objective, completing more \(0^\ell\) blocks reduces the OneJump objective. Therefore, solutions in \(\mathcal{P}\) are non-dominated to each other.

For \(x\notin\mathcal{P}\), we show that they are dominated by considering three cases.

\textbf{Case 1: \(n-k<|x|_1<n\)}. Here, the OneJump objective drops into the valley. Any solution that flips one \(1\) to \(0\) has a strictly larger OneJump objective and possibly larger ZeroRoyalRoad objective (if it completes a \(0^\ell\) block).

\textbf{Case 2: \(|x|_1 < n-k\)}. Since this solution \(x\) is not in \(\mathcal{P}\), we have \( |\{S \mid S=1^\ell\}|+|\{S \mid S=0^\ell\}|<b\), implying that there exists at least one block that mixes \(1\)s and \(0\)s. 
Here, turning a \(0\) \(1\) yields a solution that dominates \(x\). 

\textbf{Case 3: \(|x|_1 = n-k\)}. In this case, only the global optima \(1^n\) is better than \(x\) on the OneJump objective. However, like in case 2, at least one block is a mix between \(0\)s and \(1\)s. Since \(n\mod\ell=0\) by definition, there must be more than one mix block, as otherwise this assumption cannot hold. 
By swapping the \(1\)s and \(0\)s between different blocks, one can complete at least one \(0^\ell\) block (increasing the ZeroRoyalRoad) without changing the number of \(1\)s (keeping the OneJump unchanged).
This yields a solution that dominates \(x\).

For all cases, \(x\notin\mathcal{P}\) are dominated. 
\end{proof}

In the case \((n-k)\mod\ell=0\), the corresponding Pareto front is \(\bigl\{(n+k,0)\bigr\}\cup\bigl\{(i\ell+k,n-i\ell) \mid i\in\{0,1,\dots,\lfloor\frac{k}{\ell}\rfloor\}\bigr\}\).

\begin{proposition}\label{prop:OJZR-ps2}
Let \(n\in\mathbb{N}, 1<k<\frac{n}{2}\), given a bit-string of length \(n\), partitioned into \( b \) disjoint blocks \( x = S_1  S_2  \dots  S_b \), each with the same length \( \ell<k \). For the benchmark OJZR, when \((n-k)\mod\ell>0\), the Pareto optimal set is the following.
\begin{align}
\begin{split}
&\mathcal{P}=\{ \,1^n\,\} \;\,\cup\\ 
&\{ x  \mid \, |x|_1< n-k, \; |\{S \mid S=1^\ell\}|+|\{S \mid S=0^\ell\}|=b \} \, \cup \\ &\{ x \mid  |x|_1=n-k,\,|\{S\,|\,S=0^\ell\}|=\lfloor\frac{k}{\ell}\rfloor\}
\end{split}
\end{align}
\end{proposition}
\begin{proof}
Similar to Prop.\ref{prop:OJZR-ps1}, solutions in the first and the second components \(\{ \,1^n\,\} \,\cup
\{\; x \,\, \bigl| \,\, |x|_1< n-k, \; |\{S \mid S=1^\ell\}|+|\{S \mid S=0^\ell\}|=b \}\) are non-dominated to each other. 
The third component \(\{\, x \mid  |x|_1=n-k,\,|\{S\,|\,S=0^\ell\}|=\lfloor\frac{k}{\ell}\rfloor\}\) corresponds to the objective \((n-k, \lfloor\frac{k}{\ell}\rfloor\ell)\) in the objective space. 
Solutions in the third component have exactly \(\lfloor\frac{k}{\ell}\rfloor\) because there are only \(n-|x|_1=k\) zeroes available.
They are not dominated by \(1^n\) as they have a higher ZeroRoyalRoad objective. They are also not dominated by any solution in the third component (those with \(|x|_1<n-k\)) as they have higher OneJump objective with \(|x|_1=n-k\). The reason for their occurrence is explained in the main text.

For those \(x\notin\mathcal{P}\), similar to Prop.~\ref{prop:OJZR-ps1}, we may consider the cases \(n-k<|x|_1<n\), \(|x|_1<n-k\) and \(|x|_1=n-k\). The proof of the first two cases is the same as Prop.\ref{prop:OJZR-ps1}. As for the third case, the solutions left from \(\mathcal{P}\) are the \(\{x\mid |x|_1=n-k,|x|_1=n-k,\,|\{S\,|\,S=0^\ell\}|<\lfloor\frac{k}{\ell}\rfloor\}\).
It is easy to see that they have the same OneJump objective but a smaller ZeroRoyalRoad objective compared to the third component of \(\mathcal{P}\). Therefore, \(x\notin\mathcal{P}\) are all dominated.
\end{proof}

In the case \((n-k)\mod\ell>0\), the corresponding Pareto front is \(\bigl\{(n+k,0)\bigr\} \cup \bigl\{(n-k, \lfloor\frac{k}{\ell}\rfloor\ell)\bigr\} \cup \bigl\{(i\ell+k,n-i\ell) \mid i\in\{0,1,\dots,\lfloor\frac{k}{\ell}\rfloor\}\bigr\}\).

Unlike the Pareto set, the Pareto local optimal set for OJZR has only one case, as it is mainly attributed to the OneJump objective. They are the solutions right below the Pareto optimal solution at \(f_1(x)=n-k\), as shown in the following.

\begin{proposition}
With the same notation above, the local Pareto optimal solutions of OJZR are the following.
\begin{align}
\begin{split}
\mathcal{L} = \big\{ x \mid |x|_1=n-k, |\{S\mid S=0^\ell\}|<\lfloor\frac{k}{\ell}\rfloor
\big\}
\end{split}
\end{align}
\end{proposition}
\begin{proof}
Firstly, we show every solution \(x\in\mathcal{L}\) is not dominated by any of its neighbours. 
If we flip any \(1\) to \(0\), it simply reduces the OneJump objective.
If we flip any \(0\) to \(1\), it reduces the OneJump objective (as \(|x|_1>n-k\) now lies in the ``valley'').
Since no neighbour can yield improvement, \(x\in\mathcal{L}\) are not dominated by any of its neighbours.

Then, we show that \(x\notin\mathcal{P}\cup\mathcal{L}\) are dominated by at least one of its neighbours.
For solutions with \(|x|_1<n-k\), since it is not in \(\mathcal{P}\), there exists at least one mix block, meaning that there are mixes of \(1\)s and \(0\)s. 
Flipping \(0\) to \(1\) in such block can improve the OneJump objective without changing the ZeroRoyalRoad objective.
For solutions with \(n-k<|x|_1<n\), flipping any \(1\) to \(0\) increase the OneJump objective (escaping from the ``valley'') and possibly increases the second objective.
For all cases, \(x\notin\mathcal{P}\cup\mathcal{L}\) have at least one neighbour that dominates itself.
\end{proof}

\subsection{OneMax-TrailingZeroes (OMTZ)}
Now we move to the benchmarks involving OneMax. 

\begin{proposition}
Let \(n\in\mathbb{N}\), given a bit-string of length \(n\), for the benchmark OMTZ, the Pareto optimal set is the following.
\begin{align}
\mathcal{P}=\bigl\{\,0^n\,\}\,\cup\,\{ \,1^i0^{n-i} \mid i\in\{0,1,\dots,n\} \bigr\}
\end{align}
\end{proposition}
\begin{proof}
The TrailingZeroes objective enforces the structure of the Pareto set to be \(1^i0^{n-i}\). If the ``leading ones'' component does not follow this structure (i.e., having extra \(0\)s on the left), they are dominated by the solution who follow this structure. Therefore, the Pareto set of OMTZ is the same as the LeadingOnes-TrailingZeroes benchmark.
\end{proof}
The corresponding Pareto front is the same as the LeadingOnes-TrailingZeroes as well.

\subsection{OneMax-ZeroJump (OMZJ)}
\begin{proposition}
Let \(n\in\mathbb{N}, 1<k<\frac{n}{2}\), given a bit-string of length \(n\), for the benchmark OMZJ, the Pareto optimal set is the following.
\begin{align}
\mathcal{P}=\{\,0^n\,\}\,\cup\,\{ \,x \mid |x|_0\leq n-k \}
\end{align}
\end{proposition}
\begin{proof}
Since the Jump function resembles an OneMax function with a valley, the OneJump-ZeroJump function resembles an OneMinMax with some solutions no longer being Pareto optimal as they are in one of the two valleys. Similarly, for OMZJ, it resembles an OneMinMax with only one valley \(n-k<|x|_0<n\). The rest of the solutions, denoted by \(\mathcal{P}\) here, are Pareto optimal solutions of OMZJ.
\end{proof}

The corresponding Pareto front of OMZJ is \(\bigl\{\,0^n\,\bigr\}\cup\bigl\{(i,n+k-i) \mid i\in\{k,k+1,\dots,n\}\bigr\}\).

\subsection{OneMax-ZeroRoyalRoad (OMZR)}
\begin{proposition}\label{prop:OMZR-ps}
Let \(n,b,\ell\in\mathbb{N}, n=b\cdot\ell,b>1\), given a bit-string of length \(n\), partitioned into \( b \) disjoint blocks \( x = S_1  S_2  \dots  S_b \), each with the same length \( \ell \). For the benchmark OMZR, the Pareto optimal set is the following.
\begin{align}
\mathcal{P}=\{ x \,\, \bigl| \,\, |\{S \mid S=1^\ell\}|+|\{S \mid S=0^\ell\}|=b \}
\end{align}
\end{proposition}
\begin{proof}
The ZeroRoyalRoad objective enforces the block structure for the OneMax objective (block of \(1\)s) for the Pareto optimal solution. If the \(1\)s in the bit-string does not follow the block structure (i.e., break \(0^\ell\) blocks or having extra \(0\)s), they are dominated by the solutions who follow the structure. Therefore, the Pareto set of OMZR is the same as the OneRoyalRoad-ZeroRoyalRoad.
\end{proof}
The corresponding Pareto front is the same as the OneRoyalRoad-ZeroRoyalRoad as well.

\section*{Ratio of Pareto optimal solutions}
\subsection{Ratio of Pareto optimal solutions of OJZJ}\label{appendix:OJZJ}

For a OneJump-ZeroJump (OJZJ) problem, the ratio of Pareto optimal solutions is \(R(n, k) = \frac{2^n - 2\sum_{s=n-k+1}^{n-1} \binom{n}{s}}{2^n}\).
The following propositions prove that (1) for a large \(k\) near \(\frac{n}{2}\), this ratio can be very low; (2) For a smaller \(k < \frac{n}{\ln{n}}\), this ratio becomes no less than 0.5. 

\begin{proposition}\label{prop:OJZJ-r}
For the OJZJ problem with a sufficiently large \( n\in\mathbb{Z}^+ \) and a large \(k = \left\lfloor \frac{n}{2} \right\rfloor - 1\), the ratio of Pareto optimal solutions of OJZJ converges to 0 as \(n\to\infty\):
\begin{align}
\lim_{n\to\infty}R(n, k) = \lim_{n\to\infty}\frac{2^n - 2\sum_{s=n-k+1}^{n-1} \binom{n}{s}}{2^n} = 0
\end{align}
\end{proposition}

\begin{proof}
We begin by simplifying the ratio \( R(n, k) \). 
\[
\begin{split}
R(n, k) &= \frac{2^n - 2\sum_{s=n-k+1}^{n-1} \binom{n}{s}}{2^n} \\
&= 1 - \frac{2\sum_{s=n-k+1}^{n-1} \binom{n}{s}}{2^n} = 1 - \frac{\sum_{s=1}^{k-1} \binom{n}{s}}{2^{n-1}}.
\end{split}
\]

For simplicity, we choose \( k = \left\lfloor \frac{n}{2} \right\rfloor - 1 \) as a representative large value near \(\frac{n}{2}\). We analyse two cases based on the parity of \( n \).

\textbf{Case 1: \( n \) is even.} Let \( n = 2m \), where \( m \) is a positive integer,  we have \( k = m - 1 \) and
\[
\sum_{s=1}^{k-1} \binom{n}{s} = \sum_{s=1}^{m-2} \binom{2m}{s} = 2^{2m-1} - \frac{1}{2}\binom{2m}{m} - \binom{2m}{m-1} - 1
\]
Thus, the ratio becomes:
\[
\begin{split}
R(2m, m-1) &= 1 - \frac{ 2^{2m-1} - \frac{1}{2}\binom{2m}{m} - \binom{2m}{m-1} - 1 }{2^{n-1}} \\
&= \frac{\frac{1}{2}\binom{2m}{m}+\binom{2m}{m-1}+1}{2^{2m-1}}.
\end{split}
\]
Using Stirling’s approximation~\cite{feller1971introduction}, \( \binom{2m}{m} \approx \frac{4^m}{\sqrt{\pi m}} \), Therefore, for large \( m \), \(\frac{m}{m+1}\approx 1\), thus we have:

\[
R(2m, m-1) \approx \frac{(\frac{1}{2}+\frac{m}{m+1})\frac{4^m}{\sqrt{\pi m}}}{\frac{1}{2}\cdot4^{m}} \approx \frac{3\sqrt{2}}{\sqrt{\pi n}}.
\]

\textbf{Case 2: \( n \) is odd.} Let \( n = 2m + 1 \), where \( m \) is a positive integer. Then we have \( k = m - 1 \) and
\[
\begin{split}
\sum_{s=1}^{k-1} \binom{n}{s} &= \sum_{s=1}^{m-2} \binom{2m+1}{s} \\
&= 2^{2m} - \binom{2m+1}{m} - \binom{2m+1}{m-1} - 1.
\end{split}
\]
Thus, the ratio becomes:
\[
\begin{split}
R(2m+1, m-1) &= 1 - \frac{2^{2m} - \binom{2m+1}{m} - \binom{2m+1}{m-1} - 1}{2^{2m}} \\
&= \frac{\binom{2m+1}{m}+\binom{2m+1}{m-1}+1}{2^{2m}}.
\end{split}
\]
Using Stirling’s approximation again, we have:
\[
R(2m+1, m-1) \approx \frac{(\frac{2m+1}{m+1}+\frac{m}{m+2}\cdot\frac{2m+1}{m+1})\frac{4^m}{\sqrt{\pi m}}}{\frac{1}{2}\cdot4^{m}}\approx\frac{4\sqrt{2}}{\sqrt{\pi n}}.
\]

Finally, both cases converge towards 0 as \(n\to\infty\).
\end{proof}

\begin{proposition}\label{prop:OJZJ-k}
For the OJZJ problem with \( n\in\mathbb{Z}^+ \), if \(n\) is sufficiently large and the jump parameter \( k\leq\frac{n}{2}-\sqrt{\frac{n\ln4}{2}}\), we have 
\begin{align}
R(n, k) \geq 0.5
\end{align}
\end{proposition}

\begin{proof}
\( R(n, k) \geq 0.5 \) implies:
\[
\begin{split}
R(n, k) =\frac{2^n - 2\sum_{s=n-k+1}^{n-1} \binom{n}{s}}{2^n}=1-2\sum_{s=1}^{k-1}\frac{\binom{n}{s}}{2^n}\geq 0.5. 
\end{split}
\]

Treat \(\frac{\binom{n}{s}}{2^n}\) as a binomial distribution (i.e., probability of flipping \(s\) bits from \(n\) bits), we have:
\[
X\sim\text{Bin}(n,\frac{1}{2}),\qquad \text{Pr}[X=s]=\frac{\binom{n}{s}}{2^n},
\]
Thus:
\[
\begin{split}
R(n, k) &= 1 - 2
\sum_{s=1}^{k-1}\frac{\binom{n}{s}}{2^n} \qquad 
\\
&\geq 1-2\text{Pr}[X\leq k-1] \geq 0.5 \\
 & \qquad \quad\;\, \text{Pr}[X\leq k-1] \leq \frac{1}{4}.
\end{split}
\]

Then, with the Chernoff bound~\cite{mitzenmacher05} (Theorem 4.2) 
\[
\text{Pr}[X\leq k-1]\leq \exp(-\frac{\delta^2\mu}{2})
\]
Where \(\mu=\frac{n}{2}\). Rewrite \(k=\frac{n}{2}-t\) and \(k\leq(1-\delta)\mu\), we have \(\delta=\frac{2t}{n}\) and then:
\[
\begin{split}
\text{Pr}[X\leq k-1]\leq \exp(-\frac{(\frac{2t}{n})^2 \frac{n}{2}}{2})=\exp(-\frac{2t^2}{n}) &= \frac{1}{4} \\
\frac{2t^2}{n} &=\ln 4 \\
t &= \sqrt{\frac{n\ln 4}{2}}.
\end{split}
\]
And finally, we obtained the bound for \(k\):
\begin{align}
k\leq\frac{n}{2}-t=\frac{n}{2}-\sqrt{\frac{n\ln 4}{2}}.
\end{align}

Therefore, choosing any \(k\leq \frac{n}{2}-\sqrt{\frac{n\ln 4}{2}}\) makes the \(\text{Pr}[X\leq k-1]\) at most \(\frac{1}{4}\), and therefore \(R(n,k)\geq \frac{1}{2}\)

\end{proof}


\subsection{Ratio of Pareto Optimal Solutions of OJZR}\label{appendix:OJZR}

When \((n-k)\mod\ell\neq0\), the Pareto front shape of OJZR is concave and there is one Pareto optimal solution that has much more corresponding bit-strings in the decision space than the others. In this case, we show that the ratio of Pareto optimal solutions is still low in the following proposition.

Firstly, we define this ratio for OJZR. As stated in Prop.~\ref{prop:OJZR-ps2}, the Pareto set of OJZR is \(\mathcal{P}=\{ \,1^n\,\} \,\cup
\{\; x \,\, \bigl| \,\, |x|_1< n-k, \; |\{S \mid S=1^\ell\}|+|\{S \mid S=0^\ell\}|=b \}\;\cup \{\, x \mid  |x|_1=n-k,\,|\{S\,|\,S=0^\ell\}|=\lfloor\frac{k}{\ell}\rfloor\}\).
The first component has only one corresponding solution.
The second component corresponds to selecting \(i\) \(0^\ell\) blocks out of \(b\) blocks for each \(i\in\{\lceil\frac{k}{\ell}\rceil,1,\dots,m\}\). We start counting blocks from \(lceil\frac{k}{\ell}\rceil\) as a lower number of \(0^\ell\) implies more \(1\)s which already dropped into the ``valley'' of the OneJump function. The third component corresponds to the special solutions with the objective \((n-k, \lfloor\frac{k}{\ell}\rfloor\ell)\) that create the concave region in the Pareto front. It has \(\lfloor\frac{k}{\ell}\rfloor\) completed \(0^\ell\) blocks out of \(b\) blocks. 
In the rest of the \(n-\lfloor\frac{k}{\ell}\rfloor\ell\) bits, it has exactly \(n-k\) \(1\)s to be placed, yielding a total of \( \binom{b}{\lfloor \frac{k}{\ell} \rfloor} \times \binom{n - \lfloor \frac{k}{\ell} \rfloor \ell}{n - k}\) possibilities. Therefore, the ratio is given by
\begin{align}
R(n, k) = \frac{1 + \sum_{i = \lceil \frac{k}{\ell} \rceil}^{b} \binom{b}{i} + \binom{b}{\lfloor \frac{k}{\ell} \rfloor} \times \binom{n - \lfloor \frac{k}{\ell} \rfloor \ell}{n - k}}{2^n}.
\end{align}

\begin{proposition}\label{prop:OJZR}
For OJZR with problem size \( n\in\mathbb{Z}^+ \), jump parameter \( 1 \leq k < \frac{n}{2} \), and block length \( \ell \) with \( 2 \leq \ell < n \), \( n \mod \ell = 0 \) and \((n-k)\mod\ell\neq0\), 
let \( b = \frac{n}{\ell} \), the ratio of Pareto optimal solutions \(R(n, k)\) converges to 0 as \(n\to\infty\).
\end{proposition}

\begin{proof}
The numerator of \( R(n, k) \) comprises three terms.
The first term is the constant 1.
The second term \(\sum_{i = \lceil \frac{k}{\ell} \rceil}^{b} \binom{b}{i}\) is maximised when the lower limit \( \lceil \frac{k}{\ell} \rceil \) is as small as possible (i.e., 1). 
Thus:
\[
\sum_{i = \lceil \frac{k}{\ell} \rceil}^{m} \binom{b}{i} \leq \sum_{i=1}^{b} \binom{b}{i} = 2^b - 1.
\]

The third term is a product of binomial coefficients:
\[
\binom{b}{\lfloor \frac{k}{\ell} \rfloor} \times \binom{n - \lfloor \frac{k}{\ell} \rfloor \ell}{n - k }.
\]

For \( \binom{b}{\lfloor \frac{k}{\ell} \rfloor} \), to establish a general upper bound, we use:
\[
\binom{b}{\lfloor \frac{k}{\ell} \rfloor} \leq 2^b.
\]
For \( \binom{n - \lfloor \frac{k}{\ell} \rfloor \ell}{n - k } \), the maximum occurs when \( n - \lfloor \frac{k}{\ell} \rfloor \ell \) is the largest, given:
\[
\binom{n - \lfloor \frac{k}{\ell} \rfloor \ell}{n - k } \leq 2^{n - \lfloor \frac{k}{\ell} \rfloor \ell},
\]
We have:
\[
\binom{b}{\lfloor \frac{k}{\ell} \rfloor} \times \binom{n - \lfloor \frac{k}{\ell} \rfloor \ell}{n - k } \leq 2^b \times 2^{n - \lfloor \frac{k}{\ell} \rfloor \ell} = 2^{b+n-\lfloor \frac{k}{\ell} \rfloor \ell}.
\]

Combining the terms, the numerator is bounded by:
\[
\text{Numerator} \leq 1 + 2^b + 2^{b+n-\lfloor \frac{k}{\ell} \rfloor \ell}
\]

Putting it back to the ratio and substituting \(b=\frac{n}{\ell}\):
\[
R(n, k) \leq 2^{-n} + 2^{-n(1-\frac{1}{\ell})} + 2^{\frac{n}{\ell}-\lfloor \frac{k}{\ell} \rfloor \ell}.
\]

Since \(k<\frac{n}{2}\), the power of the third term
\[
\frac{n}{\ell}-\lfloor \frac{k}{\ell} \rfloor \ell < \frac{n}{\ell}-(\frac{n}{2}+1) = -1 + n(\frac{1}{\ell}-\frac{1}{2}).
\]

The first two terms are negligible as \(n\to\infty\). Therefore, 
\begin{align}
R(n, k) \leq 2^{-1 + n(\frac{1}{\ell}-\frac{1}{2})}.
\end{align}

For \(\ell>2\), that converges towards 0 as \(n\to\infty\).

However, for the edge case \(\ell=2\), this becomes 0.5.

Revisiting the ratio for case \(\ell=2\).
Since \(k<n/2\) and \((n-k)\mod 2\neq 0\) forces \(k\) to be odd, we let \(k=2p+1\).
Then, the third term of the ratio:
\[
\begin{split}
& \binom{b}{\lfloor \frac{k}{\ell} \rfloor} \times \binom{n - \lfloor \frac{k}{\ell} \rfloor \ell}{n - k} \\
&= \binom{n/2}{p}\binom{n-2p}{n-2p-1}\\ 
&= \binom{n/2}{p} \cdot (n-2p) = (n-k)\binom{n/2}{k/2}.
\end{split}
\]

Since \(\binom{n/2}{k/2}<\sum_{i=0}^{n/2}\binom{n/2}{i}=2^{n/2}\), and \((n-k)2^{n/2}\) still grows slower than the denominator \(2^n\),
thus the third term also converges towards 0 as \(n\to\infty\).

Therefore, for all the cases, the ratio converges to 0 as \(n\to\infty\).

\end{proof}

\end{document}